%% file: main_clean.tex
\documentclass[11pt]{article}

\usepackage[margin=1in]{geometry}
\usepackage{microtype}
\usepackage{amsmath, amssymb, mathtools, amsthm}
\usepackage{graphicx}
\usepackage{subcaption}
\usepackage{booktabs}
\usepackage{multirow}
\usepackage{placeins}
\usepackage{etoc}
\usepackage{enumitem}

\usepackage[authoryear, round]{natbib}
\usepackage{xcolor}
\definecolor{citegreen}{RGB}{0,128,128}
\usepackage[
    colorlinks=true,
    linkcolor=blue,
    citecolor=citegreen,
    urlcolor=blue
]{hyperref}

\usepackage[capitalize, noabbrev]{cleveref}
\crefname{assumption}{Assumption}{Assumptions}
\Crefname{assumption}{Assumption}{Assumptions}

\usepackage{autonum}

\input{macros}

\input{math_commands}

\theoremstyle{plain}
\newtheorem{theorem}{Theorem}[section]

\newtheorem{lemma}[theorem]{Lemma}
\newtheorem{corollary}[theorem]{Corollary}
\theoremstyle{definition}

\newtheorem{assumption}[theorem]{Assumption}
\theoremstyle{remark}
\newtheorem{remark}[theorem]{Remark}

\title{Demonstrations, CoT, and Prompting:\\
A Theoretical Analysis of ICL}

\author{Xuhan Tong$^{*}$ \and Yuchen Zeng$^{\dagger 1}$ \and Jiawei Zhang$^{\dagger 2}$}

\date{} %

\begin{document}
    \maketitle
    \begingroup
    \renewcommand{\thefootnote}{}

    \footnotetext{
    $^{*}$ Department of Computer Sciences, University of Wisconsin-Madison. \texttt{tong56@wisc.edu}.}
    \footnotetext{
    $^{\dagger}$ Co-last authors.
    $^{1}$ Microsoft Research. \texttt{yuchen.zeng.1998@gmail.com}.
    $^{2}$ Department of Computer Sciences, University of Wisconsin-Madison. \texttt{jzhang2924@wisc.edu}.
    }
    \footnotetext{Our code is available at: \url{https://github.com/txxxxh/Demonstrations-CoT-and-Prompting}.}
    \endgroup

    \input{abstract}

    \input{main_body}

    \bibliographystyle{plainnat}
    \bibliography{references}
    \clearpage
    \appendix
    \tableofcontents

    \clearpage
    \onecolumn

    \section{List of Common Notations}
    \input{list_notation}

    \section{Related Work}
    \input{related_work}

    \section{Proof of \cref{thm:lipschitz}}
    \input{appendix_prelim}

    \section{Formal Justification of the Padding Reduction}
    \input{appendix_padding}

    \section{Proof of \cref{thm:cot_bound}}
    \input{appendix_cot}

    \section{Proof of \cref{thm:decay_1_5}, Corollary~\ref{cor:decay_1_5} and~\ref{cor:decay_1_5_expect}}
    \input{appendix_instruction}

    \section{Proof of \cref{thm:decay_6} and Corollary~\ref{cor:decay_6}}
    \input{appendix_instruction_generalization}

    \section{Experiments}
    \input{appendix_experiment}
\end{document}

%% file: macros.tex
\usepackage{enumitem}
\usepackage{amsmath}
\usepackage{autonum}

\crefname{assumption}{Assumption}{Assumptions} \Crefname{assumption}{Assumption}{Assumptions}

 \def\norm#1{\left\lVert #1 \right\rVert} 
  
 \def\pt#1{\left( #1 \right)} 
\def\abs#1{\left| #1 \right|} \def\ith#1{^{(#1)}}

\def\mes{\operatorname{mes}} \def\dpre{D_0} \def\dicl{D}  \def\model{M}
\def\bestmodel{M_\star} \def\modelclass{\gM} \def\error{\ell_{\text{ICL}}} 
\def\xquery{x_{\text{query}}} \def\lipschitz{L} \def\yquery{y_{\text{query}}}
 \def\instruction{i} \def\iquery{i_{\text{query}}}
\def\rinstruction{\ri}  
 \def\kl{\operatorname{KL}} \def\as{\text{a.s.}}

%% file: math_commands.tex
\usepackage{amsmath, amsfonts, bm}

\def\eqref#1{(\ref{#1})}

  \def\1{\bm{1}}

  \def\ri{{\textnormal{i}}} 
 
  \def\rt{{\textnormal{t}}} 
  \def\rx{{\textnormal{x}}} \def\ry{{\textnormal{y}}}

   \def\rvi{{\mathbf{i}}}

\def\va{{\bm{a}}} \def\vb{{\bm{b}}}   
   \def\vi{{\bm{i}}}

\def\vz{{\bm{z}}}

\DeclareMathAlphabet{\mathsfit}{\encodingdefault}{\sfdefault}{m}{sl}
\SetMathAlphabet{\mathsfit}{bold}{\encodingdefault}{\sfdefault}{bx}{n}

\def\gM{{\mathcal{M}}}   
   \def\gT{{\mathcal{T}}}
 \def\gV{{\mathcal{V}}}

\def\sN{{\mathbb{N}}}  \def\sP{{\mathbb{P}}} 
\def\sR{{\mathbb{R}}}

\newcommand{\E}{\mathbb{E}}

\newcommand{\di}{\ensuremath{D_{\vi_0}}}
\newcommand{\diout}{\ensuremath{D_{\vi}}}

%% file: abstract.tex
\begin{abstract}
In-Context Learning (ICL) enables pretrained LLMs to adapt to downstream tasks by conditioning on a small set of input-output demonstrations, without any parameter updates.
Although there have been many theoretical efforts to explain how ICL works, most either rely on strong architectural or data assumptions, or fail to capture the impact of key practical factors such as demonstration selection, Chain-of-Thought (CoT) prompting, the number of demonstrations, and prompt templates.
We address this gap by establishing a theoretical analysis of ICL under mild assumptions that links these design choices to generalization behavior.
We derive an upper bound on the ICL test loss, showing that performance is governed by (i) the quality of selected demonstrations, quantified by Lipschitz constants of the ICL loss along paths connecting test prompts to pretraining samples, (ii) an intrinsic ICL capability of the pretrained model, and (iii) the degree of distribution shift. 
Within the same framework, we analyze CoT prompting as inducing a task decomposition and show that it is beneficial when demonstrations are well chosen at each substep and the resulting subtasks are easier to learn.
Finally, we characterize how ICL performance sensitivity to prompt templates varies with the number of demonstrations.
Together, our study shows that pretraining equips the model with the ability to generalize beyond observed tasks, while CoT enables the model to compose simpler subtasks into more complex ones, and demonstrations and instructions enable it to retrieve similar or complex tasks, including those that can be composed into more complex ones, jointly supporting generalization to unseen tasks. All theoretical insights are corroborated by experiments.
\end{abstract}

%% file: main_body.tex
\section{Introduction}
\label{sec:intro}

In-Context Learning (ICL)~\citep{Brown2020} is a widely used techniques that
enables Large Language Models (LLMs) to perform target tasks by presenting
multiple of input-output example pairs, termed \textit{demonstrations}, drawn from
the target tasks within the prompt. Since its introduction, ICL has demonstrated
strong empirical performance across a broad range of real-world tasks~\citep{Agarwal2024,OpenAI2023},
including coding~\citep{humaneval}, math problem solving~\citep{math,gsm8k} and
question answering~\citep{gpqa,mmlu}. Despite these empirical successes, the theoretical
understanding of ICL remains limited, as existing analyses suffer from at least
one of the two limitations described below.

First, many theoretical studies rely on strong simplifying assumptions,
including linear attention~\citep{Ahn2023,vonOswald2023,Mahankali2023,JMLR:v25:23-1042,cui2025a},
single-layer transformers or attention-only architectures~\citep{yang2024incontext,vonOswald2023,Akyurek2023,Dai2023,Mahankali2023,JMLR:v25:23-1042,cui2025a},
single-head attention~\citep{yang2024an,zhang2023trained}, or random generalized
linear regression data distributions~\citep{Ahn2023,vonOswald2023,oko2024pretrained,JMLR:v25:23-1042,Fu2024}.
While these assumptions enable tractable analysis, they abstract away architectural
and data properties that are critical in practice.

Meanwhile, intuitively, by observing demonstrations, the model can better infer
the underlying task intent that defines desirable outputs. Extensive empirical studies
have therefore focused on how to construct effective demonstrations, showing
that factors such as the number of demonstrations~\citep{Agarwal2024}, their
selection~\citep{Liu2022,su2023selective}, prompt template~\citep{ye2024investigating},
and the inclusion of Chain-of-Thought (CoT~\citep{Wei2022}) reasoning in demonstrations
can lead to substantial differences in performance. However, most theoretical work
fail to depict how these key practical factors affect ICL performance~\citep{Feng2023,li2023dissecting,tutunov2024}.
To bridge this gap, in this paper, we present a set of theoretical results under
mild assumptions, all of which are corroborated by corresponding synthetic
experiments.

\paragraph{Our Contributions.}

Under the above formulation, we provide a theoretical characterization of how demonstration
selection, CoT, prompt template, number of demonstrations influences ICL
performances. Our main contributions are summarized as follows.
\vspace{-0.1em}
\begin{itemize}[leftmargin=*, topsep=0pt, itemsep=1pt, parsep=2pt]
    \item \textbf{Quantifying Demonstration Effectiveness via Lipschitz
        Constants.} We prove that when ICL is used to adapt a pretrained model
        to a downstream task with an unseen test prompt, the performance is governed
        by three factors. The first is the effectiveness of the demonstration in
        identifying the target task, captured by the minimum Lipschitz constant of
        the ICL loss along a path that maps the test prompt to any sample observed
        during pretraining. A large Lipschitz constant indicates that the
        demonstrations fail to stably pin down the underlying task. In addition,
        the analysis also captures the effects of the pretrained model’s
        intrinsic ICL capability and the distance between the test prompt and
        pretraining prompts, which reflects the degree of distribution shift.

    \item \textbf{Characterizing When CoT Benefits ICL.} We next consider the
        setting where demonstrations include CoT traces, which can be viewed as decomposing
        the target task into multiple subtasks. We show that the loss on the
        test prompt can be decomposed into a weighted sum of the losses incurred
        when performing ICL on these subtasks, where the weights are determined by
        the Lipschitz constants of the loss along paths that map subtask-specific
        samples in the test prompt to the corresponding subtask samples in the pretraining
        dataset. This result indicates that CoT is beneficial to ICL when the
        task decomposition selects well-learned subproblems and when the
        demonstrations for each subtask are chosen such that the subtasks are
        well-defined and reliably identifiable.

    \item \textbf{Establishing the Interaction Between Prompt Templates and
        Demonstration Count.} We show that the influence of prompt templates on
        ICL exhibits a mixed effect that depends on the number of demonstrations.
        By analyzing changes in output logits with respect to the prompt template,
        we prove that this influence is generally upper bounded by an
        exponential decay in the number of demonstrations. Consequently, when
        the number of demonstrations is small, prompt templates can substantially
        affect output probabilities, whereas as the number of demonstrations
        increases, ICL performance becomes increasingly invariant to the choice of
        prompt template. An important exception arises when prompt templates
        provide different incorrect instructions across demonstrations, in which
        case this decay no longer holds.
\end{itemize}
All of our theoretical results are supported by experiments.

\vspace{1em}
Our results can be understood through the model's ability to generalize beyond pretraining. First, we show that pretraining enables the model to generalize to tasks that are similar to those seen during training. Building on this, demonstrations in a prompt enable the model to retrieve such tasks during ICL: given sufficiently informative examples, the model can map the test prompt to tasks that were either seen during pretraining, similar to them, or composed from them, with the retrieval becoming more reliable as the number of demonstrations increases. Furthermore, CoT extends this capability by enabling generalization to more complex tasks: if a task can be decomposed into subtasks that are well learned during pretraining, then the model can combine them to solve the overall task. Together, these results show that ICL leverages pretraining to generalize not only to similar tasks, but also to their compositions, with demonstrations and CoT serving as mechanisms for retrieving and composing such knowledge.

\vspace{-.1in}
\section{Lipschitz Characterization of Demonstration Effectiveness in ICL}
\label{sec:general}
\subsection{Background}
\label{subsec:general_formulation} We start by presenting our problem setting
for analyzing ICL with a test prompt unseen through pretraining. Let $\gV$ denote
a finite vocabulary, and let $\gV^{L}$ denote the set of all symbol sequences of
length at most $L$. Since $\gV$ is finite and the sequence length is bounded,
$\gV^{L}$ is a finite (and hence compact) set. We write $\gV^{*}:= \gV^{L}$. A
data sample is denoted by $(x,y)$, where $x, y \in \gV^{*}$.
Let $\modelclass$ denote a class of Transformer models, where each $\model \in \modelclass$
is a function $\model : \gV^{*}\to \gV^{*}$ that maps an input prompt sequence to
its continuation. We assume the existence of a target model
$\bestmodel : \gV^{*}\to \gV^{*}$ that is optimized for performing ICL on a
given target task. Such a target model is commonly assumed in theoretical analyses~\citep{zeng2024the,pmlr-v195-giannou23a}.
This assumption is realistic because the Transformer class is a universal sequence-to-sequence
approximator~\citep{Yun2020} and can represent a broad range of target tasks. Moreover,
introducing $\bestmodel$ lets us carry out the analysis within a well-defined
hypothesis space $\modelclass$, rather than imposing additional structural assumptions
on the target task.

We consider a pretrained Transformer model $\model \in \modelclass$, and denote
the test prompt as $\vz = (x_{1}, y_{1}, \ldots, x_{n},y_{n},\xquery)$, where $n
\in \sN$ is the number of demonstrations and $\xquery \in \gV^{*}$ is the query.
The \textit{ICL loss} is defined as
$\error (\vz) := \abs{\model(\vz)- \bestmodel(\vz)}.$
Finally, let domain $\dpre \subset \gV^{*}$ denote the set of pretraining
prompts that contain a variable number of demonstrations.

To establish our main results, we will rely on two classical tools from
approximation theory, which we recall below.
\begin{lemma}[Remez Inequality for Polynomials~\citep{remez1936}]
    \label{thm:remez} If $J\in\sR$ is a finite interval, and $E \subset J$ is an
    an arbitrary measurable set, then for any polynomial $p$ of degree $n$,
    \begin{equation}
        \max_{x \in J}\abs{p(x)}\leq \pt{\frac{4 \mes (J)}{\mes (E)}}^{n}\sup_{x\in
        E}\abs{p(x)}.
    \end{equation}
\end{lemma}

\begin{lemma}[Bernstein Approximation Theorem, Lipschitz Case~\citep{Lorentz1986}]
    \label{thm:berstein} Let $f: [0,1] \to \sR$ be a Lipschitz continuous function
    with Lipschitz constant $\lipschitz$. Let $B_{n}f$ denote the degree-$n$
    Bernstein polynomail (see Appendix~\ref{app:lipschitz} for definition)
    approximation of $f$. Then, for all $n\geq 1$,
    \begin{equation}
        \sup_{x\in[0,1]}\abs{B_n f(x) - f(x)}\leq \frac{\lipschitz}{2\sqrt{n}}.
    \end{equation}
\end{lemma}
\vspace{-0.8em}
\subsection{Lipschitz Generalization Bound}
\label{sec:result_lipschitz} We now combine the above ingredients to characterize
how effectively in-context demonstrations identify the target task and control generalization
to unseen prompts.
\begin{theorem}[(Informal) ICL Generalization Bound]
    \label{thm:lipschitz} For any test prompt $\vz \in \dicl \setminus \dpre$ that
    is unseen during pretraining, the ICL loss satisfies
    \begin{equation}
        \label{eq:lipschitz}\error(\vz) \leq O(\kappa^{\lipschitz^{2}})\sup_{\vz'
        \in \dpre}\error(\vz '),
    \end{equation}
    where $\lipschitz$ is an effective Lipschitz parameter determined by the
    smallest Lipschitz constant of $\error$ along straight-line paths connecting
    the test prompt $\vz$ to pretraining prompts $\vz' \in \dpre$, and $\kappa >
    1$ quantifies the degree of distribution shift.
\end{theorem}
\begin{proof}[Proof Sketch]
    The key idea of this proof is to use Remez-Chebyshev inequality (Lemma~\ref{thm:remez}).
    The main difficulty is that the loss $\error(\cdot)$ is \emph{not} a polynomial
    and is defined over the prompt space rather than a one-dimensional interval.
    Here we introduce two reductions before applying Lemma~\ref{thm:remez}.

    \begin{enumerate}[leftmargin=*, topsep=0pt, itemsep=0pt, parsep=0pt]
        \item \textit{Step 1: Reduce the prompt space to a finite interval via a
            path.} Fix an unseen test prompt $\vz \notin \dpre$ and pick any
            $\vz' \in \dpre$. Consider the straight-line path $\gamma:[0,1]\to \gV
            ^{*}$,
            \[
                \gamma(t)=\vz' + t(\vz-\vz'),
            \]
            so that $\gamma(0)=\vz'$ and $\gamma(1)=\vz$. Restricting the loss to
            this path yields a function on a compact interval,
            \[
                \error^{\gamma}(t) \triangleq \error(\gamma(t)), \qquad t\in[0,1]
                .
            \]
            This step converts the domain requirement of Remez--Chebyshev (a
            finite interval) into our setting.

        \item \textit{Step 2: Convert a Lipschitz function into a polynomial via
            Bernstein approximation.} Remez--Chebyshev applies to polynomials,
            so we approximate the (generally non-polynomial) function
            $\error^{\gamma}$ by a degree-$n$ Bernstein polynomial
            $B_{n}\error^{\gamma}$. Since $\error^{\gamma}$ is Lipschitz on $[0,
            1]$ with constant denoted by $\lipschitz_{\gamma}$, applying Lemma~ \ref{thm:berstein}
            gives
            \begin{equation}
                \label{eq:bnerror}\sup_{t\in[0,1]}\bigl|B_{n}\error^{\gamma}(t)-
                \error^{\gamma}(t)\bigr| \le \frac{\lipschitz_{\gamma}}{2\sqrt{n}}
                .
            \end{equation}
            This is the key trick that allows us to manufacture a polynomial while
            controlling the approximation error explicitly.

        \item \textit{Step 3: Apply Remez-Chebyshev to the polynomial surrogate.}
            We now apply Lemma~\ref{thm:remez} to the polynomial
            $B_{n}\error^{\gamma}$ on the interval $[0,1]$, with any subset $E\in
            [0,1]$ satisfying
            \[
                E \supset \{t\in[0,1] : \gamma(t)\in\dpre\},
            \]
            and has positive measure. This yields
            \begin{equation}
                \label{eq:bnbound}\max_{t\in[0,1]}|B_{n}\error^{\gamma}(t)| \le C
                _{\gamma}^{n}\, \sup_{t\in E}|B_{n}\error^{\gamma}(t)|
            \end{equation}
            for a universal constant $C_{\gamma}>1$ (equivalently,
            $C_{\gamma}= 4\,\mes([0,1])/\mes(E)$ in the statement of Lemma~\ref{thm:remez}).

        \item \textit{Step 4: Transfer the bound back to $\error$ and choose $n$.}
            Combining the \cref{eq:bnerror} and \cref{eq:bnbound} yields
            \[
                \max_{t\in[0,1]}\error^{\gamma}(t) \le C_{\gamma}^{n}\sup_{t\in E}
                \error^{\gamma}(t) + O\left(\frac{\lipschitz_{\gamma}}{\sqrt{n}}\right
                ).
            \]
            Finally, evaluating at $t=1$ gives
            the desired bound for $\error(\vz)=\error^{\gamma}(1)$, and taking a
            optimal $\gamma$ by selecting over $\vz'\in\dpre$ completes the
            proof. Complete details are provided in Appendix~\ref{app:lipschitz}.
    \end{enumerate}
\end{proof}

\begin{remark}
    Our analysis connects a test query $\vz \in \dicl$ to a pretraining prompt $\vz
    ' \in \dpre$, which implicitly assumes compatible sequence lengths. In practice,
    pretraining sequences are often longer than ICL queries. We therefore allow $\vz$
    to be extended by \emph{virtual padding} (also used in \cref{thm:cot_bound} and
    \cref{subsec:instruction_setup}); a formal justification is given in Appendix~\ref{app:padding_general}.
    Moreover, our setup focuses on greedy decoding (temperature $=0$), where the
    model output is a single token. More general settings that account for distributional
    outputs are deferred to Section~\ref{sec:instruction}.
\end{remark}

\cref{thm:lipschitz} suggests that ICL performance on an unseen prompt is
governed by three factors: (i) the model's intrinsic ICL capability, captured by
$\sup_{\vz' \in \dpre}|\error(\vz')|$; (ii) an effective Lipschitz parameter $\lipschitz$
that quantifies how stably the demonstrations identify the underlying task; and
(iii) prompt distribution, characterized by $\kappa$.
expansion. The first factor is straightforward: even if the demonstrations
perfectly specify the task, a model with weak intrinsic ICL ability (interpreted
as \emph{large pretraining error}) may still incur large loss. The second factor
is more subtle and reflects \emph{task ambiguity} induced by the demonstrations.
The third factor captures how far the test prompt deviates from pretraining prompts.

Intuitively, $\lipschitz$ measures how rapidly the loss can change as we move
from the test prompt toward the pretraining prompts. When the demonstrations do not
uniquely pin down a task, small perturbations of the prompt can switch the most
plausible task interpretation, which in turn can lead to a large change in the model's
predicted output distribution and the resulting loss. In this regime, the loss landscape
is ``steep'' along prompt-space paths, corresponding to a large effective
Lipschitz parameter.

A concrete example illustrates this phenomenon. Consider the prompt
\[
    \texttt{Japan}\to\texttt{Tokyo},\ \texttt{France}\to\texttt{Paris},\ \texttt{
     USA}\to ?
\]
This context is compatible with multiple latent task hypotheses, such as mapping
countries to their capitals or mapping countries to their largest cities. Under different
hypotheses, the appropriate continuation could be either \texttt{Washington DC} or
\texttt{New York}. As a result, small perturbations of the prompt—such as
modifying or replacing a single demonstration—can shift which hypothesis best
explains the context, leading to a large change in the model’s continuation distribution
for \texttt{USA}$\to$. This behavior corresponds to a large effective Lipschitz
$\lipschitz$: the loss can vary sharply because the demonstrations fail to robustly
resolve task ambiguity.

In contrast, consider the prompt
\[
    \footnotesize \texttt{Australia}\to\texttt{Sydney},\ \texttt{Turkey}\to\texttt
    {Istanbul},\ \texttt{USA}\to ?
\]
In this case, the demonstrations more strongly favor a single underlying task interpretation,
such as mapping countries to their largest cities, while disfavoring alternative
hypotheses (e.g., mapping to capitals). As a result, the task identity is more clearly
specified, and small variations of the prompt are unlikely to change the
inferred task.

This corresponds to a smaller effective Lipschitz $\lipschitz$, indicating that
the demonstrations provide a more stable alignment between the unseen prompt and
a coherent task seen during pretraining.
\begin{remark}
This result shows that pretraining equips the model with an intrinsic ability to generalize from observed tasks to related tasks, thereby inducing a structured task space beyond those explicitly seen during training. 
\end{remark}

In \cref{sec:cot}, we will show how this Lipschitz parameter governs when CoT can improve
ICL. 
In \cref{sec:instruction}, we will further show that Lipschitz continuity characterizes
the magnitude of variation between different prompt templates. A small Lipschitz
guarantees uniformly small variation of function value along the path, and hence
strong generalization. This distinction explains why certain CoT constructions
or prompt formats may succeed or fail to generalize reliably across prompts.

\section{CoT in ICL: A Lipschitz-Based Perspective}
\label{sec:cot}

\subsection{Problem Setup}

Chain-of-thought (CoT) prompting~\citep{Wei2022} is a powerful technique for eliciting
multi-step reasoning in large language models by including intermediate reasoning
steps in the demonstrations. In this section, we analyze how CoT prompting influences
ICL performance. To this end, we extend the notation introduced in \cref{subsec:general_formulation}
to account for settings in which intermediate steps are explicitly provided.

We assume that the target task can be decomposed into $K$ sequential subtasks.
Accordingly, we represent the input as $x = x^{(0)}$, the final output as
$y = x^{(K)}$, and the intermediate reasoning steps as $x^{(1)}, \ldots, x^{(K-1)}$.
Each data sample is thus denoted by the sequence $(x^{(0)}, \ldots, x^{(K)})$.

Under this formulation, the test prompt without generating any new intermediate
steps is
\[
    \vz^{(0)}= \vz = (x^{(0)}_{1}, \ldots, x^{(K)}_{1}, \ldots, x^{(0)}_{n}, \ldots
    , x^{(K)}_{n}, \xquery^{(0)}),
\]
and the prompt augmented with the first $k=1,\ldots,K-1$ intermediate steps is
denoted by
\[
    \vz^{(k)}= (x^{(0)}_{1}, \ldots, x^{(K)}_{1}, \ldots, x^{(0)}_{n}, \ldots, x^{(K)}
    _{n}, \xquery^{(0)}, \ldots, \xquery^{(k)}).
\]

Accordingly, for each step $k \in{0,\ldots,K-1}$, let $\dpre^{(k)}$ denote the set
of pretraining prompts that contain $n$ demonstrations and $k$ intermediate steps.

\subsection{ICL Generalization Bound in the Presence of CoT}
\label{subsec:cot_bound}

\begin{theorem}[(Informal) ICL Generalization Bound with CoT Prompting]
    \label{thm:cot_bound} For a test ICL prompt $\vz$ with CoT demonstrations consisting
    of $K$ steps, the ICL loss satisfies
    \begin{equation}
        \error (\vz) \leq \sum_{k=0}^{K-1}O\pt{\exp(\lipschitz^2_{(k)})}\cdot \sup
        _{\vz\ith{k} \in \dpre\ith{k}}\error(\vz\ith{k}),
    \end{equation}
    where $\lipschitz_{(k)}$ denotes the effective Lipschitz parameter
    associated with the $k$-th subtask.
\end{theorem}

\cref{thm:cot_bound} formalizes the role of CoT in ICL by decomposing its effect
into two theoretically grounded components. The bound separates the contribution
of the model’s intrinsic ICL capability at each intermediate step, measured by
$\sup_{\vz^{(k)} \in \dpre^{(k)}}\error(\vz^{(k)})$, from the contribution of the
effective Lipschitz parameters $\lipschitz_{(k)}$ induced by the task decomposition.
In practice, modern pretrained models often already exhibit strong intrinsic CoT
capability, so that $\sup_{\vz^{(k)} \in \dpre^{(k)}}\error( \vz^{(k)})$ is
comparable to the no-CoT case $\sup_{\vz \in \dpre}\error(\vz)$. Under this regime,
the theory predicts that the primary benefit of CoT arises from its ability to
reduce the effective Lipschitz parameters of the induced subtasks. As discussed in
\cref{sec:result_lipschitz}, these parameters characterize how sensitively the
ICL loss varies along prompt-space paths. With CoT, $\lipschitz_{(k)}$ depends
on whether each subtask is well aligned with the pretrained distribution and
whether the demonstrations at that step stably identify the subtask. This explains
how CoT improves generalization: it succeeds when it reshapes a complex task
into subtasks that are both familiar to the model and more specified by the
demonstrations. Complete proof is provided in Appendix~\ref{app:cot_bound_proof}.
\begin{remark}
This result further shows that the pretrained model can extend beyond individual tasks to their compositions (see \cref{subsec:experiment_cot}). When a complex task can be decomposed into subtasks that are well learned during pretraining, the model can generalize to such compositional tasks through multi-step reasoning.
\end{remark}
\vspace{-0.4em}
\section{Synergistic Effects of Demonstration Count and Prompt Templates in ICL}
\label{sec:instruction} In this section, we study how the number of demonstrations
and the choice of prompt templates jointly affect ICL.

\subsection{Problem Setup}
\label{subsec:instruction_setup} We consider ICL prompts with varying numbers of
demonstrations $N$ and varying prompt templates, where each prompt template is
instantiated as an instruction $\instruction \in \gV^{*}$.

\paragraph{Latent Task Model.}
We assume data are generated from a latent task variable $t \in \gT$, where
$\gT$ is a finite set of tasks. Each task $t$ induces a task-conditional data-generating
process. Let $P(\cdot)$ to denote probability distributions and $p(\cdot)$ for
the corresponding probability mass or density functions. Conditioned on $t$,
inputs and outputs are generated as
\begin{align}
    \rx \sim P_{\rx \mid t}, \qquad \ry = f_{t}(\rx),
\end{align}
where $f_{t}: \gV^{*}\to \gV^{*}$ is a (possibly unknown) task-specific input-output
mapping. We assume a strictly positive prior over tasks,
\[
    \pi(t) := \sP(\rt = t) > 0 \qquad \text{for all }t \in \gT.
\]

\vspace{-0.3em}
\paragraph{Instructions.}
A single task may admit multiple valid instructions. We model instructions as
random variables drawn from a task-conditional distribution
$\rinstruction \sim P_{\rinstruction \mid t}.$ For a target task $t^{*}\in \gT$,
we say that an instruction $\instruction$ is \emph{correct} for $t^{*}$ if
\begin{equation}
    p(\instruction \mid t^{*}) > p(\instruction \mid t) \quad \text{for all }t \in
    \gT \setminus \{t^{*}\}.
\end{equation}
Otherwise, the instruction is said to be \emph{incorrect} for $t^{*}$.
\vspace{-0.5em}
\paragraph{Prompt Formats.}
An ICL prompt consists of demonstrations followed by a query. Depending on how instructions
appear in the prompt, we consider the following formats.
\begin{enumerate}[leftmargin=*, topsep=0pt, itemsep=1pt]
    \item Single correct instruction (appears once).
        \[
            \vz = (\instruction, (x_{1},y_{1}), \ldots, (x_{N},y_{N}), \xquery).
        \]
        Example: \emph{Please add the two numbers.} $1,2\to3; 4,5\to9; 6,7\to?$.

    \item Single incorrect instruction (appears once or absent).
        \[
            \vz = (\tilde{\instruction}, (x_{1},y_{1}), \ldots, (x_{N},y_{N}), \xquery
            ),
        \]
        where $\tilde{\instruction}$ is input $i$ or a virtual token
        $\langle\text{pad}\rangle$ .

        Example: \emph{Please subtract the two numbers.} $1,2\to 3 ; 4,5\to9; 6,7
        \to ?$.

    \item Repeated correct instruction (shared).
        \[
            \vz = ((\instruction,x_{1},y_{1}), \ldots, (\instruction,x_{N},y_{N})
            , (\instruction,\xquery)).
        \]
        Example: $1{+}2\to3; 4{+}5\to9; 6{+}7\to?$.

    \item Repeated incorrect instruction (shared).
        \[
            \vz = ((\instruction,x_{1},y_{1}), \ldots, (\instruction,x_{N},y_{N})
            , (\instruction,\xquery)).
        \]
        Example: $1{-}2\to3; 4{-}5\to9; 6{-}7\to?$.

    \item Varying correct instructions (one per demonstration).
        \[
            \vz = ((\instruction_{1},x_{1},y_{1}), \ldots, (\instruction_{N},x_{N}
            ,y_{N}), (\iquery,\xquery)).
        \]
        Example: $1{+}2\to3; \text{sum}(4,5)\to9; 6\ \text{plus}\ 7 \to?$.

    \item Varying incorrect instructions (one per demonstration).
        \[
            \vz = ((\instruction_{1},x_{1},y_{1}), \ldots, (\instruction_{N},x_{N}
            ,y_{N}), (\iquery,\xquery)).
        \]
        Example: $1{-}2\to3; \text{product}(4,5)\to9; 6\ \text{mod}\ 7 \to?$.
\end{enumerate}
For notational convenience, we use $\vi$ to denote the collection of all
instructions appearing in a prompt, regardless of the specific prompt format.
Unless otherwise specified, we assume that input-output pairs $(\rx,\ry)$ are drawn
i.i.d., and that instructions $\rvi$ are drawn i.i.d. whenever more than one
instruction appears in the prompt, independently of the input-output pairs $(\rx,
\ry)$.

\paragraph{Pretraining via cross-entropy recovers the Bayesian predictor.}
Under the latent-task framework introduced above, pretraining can be viewed as learning
a conditional predictor $q(\cdot \mid \vz)$ by minimizing the empirical cross-entropy
loss over a dataset $\{(\vz^{(m)}, y^{(m)})\}_{m=1}^{M}$:
\begin{equation}
    q^{*}(\vz)\in \arg\min_{q(\cdot \mid \vz)}\frac{1}{M}\sum_{m=1}^{M}\Big[-\log
    q\big (y^{(m)}\mid \vz^{(m)}\big)\Big]. \label{eq:ce}
\end{equation}
In the population limit $M \to \infty$, this objective converges to minimizing the
expected log-loss $\mathbb{E}_{(\vz,y)}\left[-\log q(y \mid \vz)\right]$ over all
conditional predictors $q(\cdot \mid \vz)$. It is well known that the unique
minimizer of this objective is the Bayesian posterior-predictive distribution,
\begin{equation}
    p(\yquery \mid \vz) = \sum_{t \in \gT}p(\yquery \mid \xquery, t)\, p(t \mid \vz
    ), \label{eq:posterior_predictive}
\end{equation}
which is Bayes-optimal under the log-loss. Accordingly, throughout our analysis we
assume that idealized pretraining recovers this Bayesian predictor at the level of
conditional distributions, up to standard approximation error.

\subsection{When Sensitivity to Prompt Templates Vanishes with Demonstrations}
To analyze how the influence of prompt templates evolves as the number of demonstrations
increases, we require a mild identifiability condition on the latent tasks. This
assumption ensures that different tasks induce distinguishable data-generating distributions,
allowing the model to infer the underlying task from sufficiently many
demonstrations.
\begin{assumption}
    [Identifiability of Tasks]\label[assumption]{assp:identifiability} For any
    two distinct tasks $t_{1}, t_{2}\in \gT$ with $t_{1}\neq t_{2}$, the corresponding
    task-conditional joint distributions are strictly distinguishable, in the
    sense that $\kl\pt{P_{\rx, \ry \mid t_1} \,\Vert\, P_{\rx, \ry \mid t_2}}
    > 0.$
\end{assumption}
\begin{assumption}
    [Bounded Instruction Sensitivity]\label[assumption]{assp:bounded-instruction}
    There exists a constant $L_{\vi}>0$ such that for all $t \in \gT$, $\|\nabla_{\vi}
    \log p(\vi \mid t)\| \le L_{\vi}$.
\end{assumption}
Based on this assumption, we establish the following result, which shows that the
sensitivity of the model output to prompt templates decays exponentially with the
number of demonstrations for prompt formats 1-5.

\begin{theorem}[Exponential Convergence of Posterior Predictive (Formats 1--5)]
    \label{thm:decay_1_5} Consider prompt formats 1--5. Under
    \cref{assp:identifiability,assp:bounded-instruction}, there exist constants $\alpha
    ,\beta>0$ such that, with high probability, for all sufficiently large numbers
    of demonstrations $N$,
    \begin{align}
        \bigl| p(\yquery \mid \vz) - p(\yquery \mid \xquery, t^{*}) \bigr| \le \beta \exp(-\alpha N).
    \end{align}
    Equivalently, for any two instructions $\vi,\vi'$ (with all non-instruction
    parts of $\vz$ fixed), we have
    \begin{align}
        \bigl| p(\yquery \mid \vz(\vi)) - p(\yquery \mid \vz(\vi')) \bigr| \le 2\beta \exp(-\alpha N).
    \end{align}
\end{theorem}

\begin{corollary}[Gradient Decay of Prompt Sensitivity]
    \label{cor:decay_1_5} Under the conditions of \cref{thm:decay_1_5}, then
    there exist constants $\alpha',\beta'>0$ such that, with high probability,
    for all sufficiently large $N$,
    \begin{align}
        \bigl\|\nabla_{\vi}p(\yquery \mid \vz)\bigr\| \le \beta' \exp(-\alpha' N).
    \end{align}
\end{corollary}

\begin{corollary}[Expected Instruction Stability]
    \label{cor:decay_1_5_expect} Fix a demonstration count $N$ and consider prompts
    of formats~1--5 with exactly $N$ demonstrations. Let $q_{\theta}(\vz):=p_{\theta}
    (\yquery \mid \vz)$ be a predictor learned by cross-entropy training. There
    exists an $\varepsilon_{N}\ge 0$ such that, with high probability, for all
    sufficiently large $N$ and all $\vi,\vi' \in\dicl$,
    \begin{align}
        \label{eq:decay_1_5_expect} & \mathbb{E}_{\vz}\Big[ \big|q_{\theta}(\vz(\vi) ) - q_{\theta}(\vz(\vi'))\big| \Big]                                                                                                   \\
                                    & ~\le~ \underbrace{2\beta e^{-\alpha N}}_{\text{demonstrations effect}}+ \underbrace{ 2\mathbb{E}_{\vz}\Big[\big|q_{\theta}(\vz) - q^{*}(\vz)\big|\Big] }_{\text{pretraining effect}}.
    \end{align}
\end{corollary}
\begin{remark}
    \cref{eq:decay_1_5_expect} decomposes instruction sensitivity into task identification
    from demonstrations and the discrepancy between $q_{\theta}$ and $q^{*}(\vz)$.
    The first term captures the effect of demonstrations and decays
    exponentially with the number of demonstrations $N$, as established in \cref{thm:decay_1_5}.
    The second term measures the expected discrepancy between the learned
    predictor $q_{\theta}$ and the Bayesian predictor $q^{*}$, reflecting finite
    pretraining and optimization effects, and is controlled by how well the
    chosen model class can approximate the Bayesian predictor under the training
    distribution.
\end{remark}

\cref{thm:decay_1_5} provides a theoretical explanation for a commonly observed empirical
phenomenon: the influence of prompt templates tends to diminish as the number of
demonstrations increases. In the low-data regime, prompt templates can have a
substantial effect on the model’s behavior; however, as more demonstrations are
provided, their impact is rapidly suppressed by the evidence contained in the demonstrations
themselves.

Importantly, the theorem applies to a broad range of practical prompting
scenarios. It covers settings in which the instruction is correct, incorrect (e.g.,
systematically replacing $+$ with $-$), or even absent. In all such cases, a sufficient
number of demonstrations allows the model to reliably infer the underlying task and
produce correct predictions, despite potentially misleading prompt templates. The
theorem does not, however, extend to prompt format~6, which represents a
fundamentally more challenging regime.

In practice, much of the data we encounter has exactly the structure discussed in
this section: an instruction paired with a corresponding input and output. The instruction
itself may be different phrasings of the same underlying task. Therefore, training
on data with this structure also covers a broad range of scenarios. Complete
details are provided in Appendix~\ref{app:instruction}.

\subsection{When Sensitivity to Prompt Templates Does not Vanishes with
Demonstrations}
We now turn to prompt format~6, and examine how prompt sensitivity behaves in this
case.

\paragraph{Problem setup.}
We consider $\vi$ in a \emph{$d$-dimensional Euclidean instruction space} $\gV^{*}$,
whose dimension scales linearly with the number of demonstrations. All gradients
and norms with respect to instruction variables are taken with respect to this Euclidean
structure.

Here we distinguish two instruction domains. We denote by $\di \subset \gV^{*}$
an \emph{inner instruction domain}, corresponding to the set of instruction
variables induced by prompt formats 1-5. We further consider an extended
instruction domain $\diout \supset \di$, corresponding to prompt format 6. For analytical
convenience, we characterize these domains geometrically by Euclidean balls: there
exist radius $0<r<R<\infty$ such that $B_{r}\subset \di \subset \diout \subset B_{R}$,
where $B_{r}$ and $B_{R}$ denote Euclidean balls in $\gV^{*}$ with radius $r$
and $R$.

\paragraph{Sensitivity to Prompt Templates.}
Fix all non-instruction components of the prompt $\vz$, view $f(\vi) := p(\yquery
\mid \vz)$ as a scalar-valued function defined on $\gV^{*}$, and let
$f^{*}:= p(\yquery \mid \xquery, t^{*})$. We now present our result for prompt format
6.

\begin{theorem}[(Informal) Failure of Convergence of Posterior Predictive (Format
6)]
    \label{thm:decay_6} Consider prompt format~6, under
    \cref{assp:identifiability,assp:bounded-instruction}, there exist constants $\alpha
    ,\beta>0$ and $\epsilon>0$ such that, with high probability, for all sufficiently
    large $N$,
    \begin{align}
        \label{eq:decay6_true_task_form}\bigl| f(\vi) - f^{*}\bigr| \le O\bigl(\exp( L_{\vi}^{2})\bigr)\bigl(\beta e^{-\alpha N}+\epsilon\bigr).
    \end{align}
    Equivalently, for any two instruction sequences $\vi,\vi'$ (with all non-instruction
    parts of $\vz$ fixed), we have
    \begin{align}
        \label{eq:decay6_two_instruction_form}\bigl| f(\vi) - f(\vi') \bigr| \le 2 \cdot O\bigl(\exp( L_{\vi}^{2})\bigr)\,\bigl(\beta e^{-\alpha N}+\epsilon\bigr).
    \end{align}
    where $L_{\vi}$ is the effective instruction-Lipschitz parameter. In particular,
    the bound does not vanish as $N\to\infty$ due to the non-decaying term
    $\epsilon$.
\end{theorem}

\begin{corollary}[(Informal) Failure of Prompt Sensitivity Decay (format 6)]
    \label{cor:decay_6} Under the conditions of \cref{thm:decay_6}, there exist
    constants $\alpha',\beta'>0$ and $\epsilon'>0$ such that, with high probability,
    for all sufficiently large $N$,
    \begin{align}
        \sup_{\vi\in \diout}\|\nabla_{\vi}f(\vi)\| \le O\bigl(L_{\vi}^{4}\cdot \exp( L_{\vi}^{2})\bigr)\bigl(\beta' e^{-\alpha' N}+ \epsilon'\bigr),
    \end{align}
    where $L_{\vi}$ is the effective Lipschitz parameter implied by \cref{assp:bounded-instruction}.
\end{corollary}

\cref{thm:decay_6} provides a theoretical explanation for the empirical observation
that prompt sensitivity may persist even as the number of demonstrations increases,
particularly when instructions vary across demonstrations (format 6). In this setting,
the model may struggle to reconcile conflicting instructions, leading to poor performance
on the target task.

Both Corollary~\ref{cor:decay_1_5} and~\ref{cor:decay_6} concretely instantiate the
\emph{Lipschitz intuition} discussed in \cref{sec:result_lipschitz}. In the present
setting, the relevant Lipschitz constant corresponds to the effective Lipschitz
continuity of $p(\yquery\mid \vz)$ over domain $\diout$. When the number of demonstrations
$N$ is large, Corollary~\ref{cor:decay_1_5} shows that the gradient of $f$ goes
to zero. This implies a small effective Lipschitz constant along the instruction-induced
paths, which (by \cref{thm:lipschitz}) leads to strong generalization. In contrast,
Corollary~\ref{cor:decay_6} demonstrates that for prompt format~6 the
instruction gradient remains large. This corresponds to a large effective Lipschitz
constant which prevent meaningful generalization. Complete details are provided
in Appendix~\ref{app:instruction-generalization}.
\begin{remark}
Combining the above results, ICL can be understood as a retrieval process over the pretraining-induced task space. Given demonstrations in the prompt, the model can identify tasks that are (i) seen during pretraining, (ii) similar to them, or (iii) compositions of such subtasks. As the number of demonstrations increases, this retrieval becomes more reliable, enabling the model to select and assemble the appropriate task for the test input.
\end{remark}


\begin{figure*}[t]
    \centering
    \begin{subfigure}
        [t]{0.32\textwidth}
        \centering
        \includegraphics[width=\textwidth]{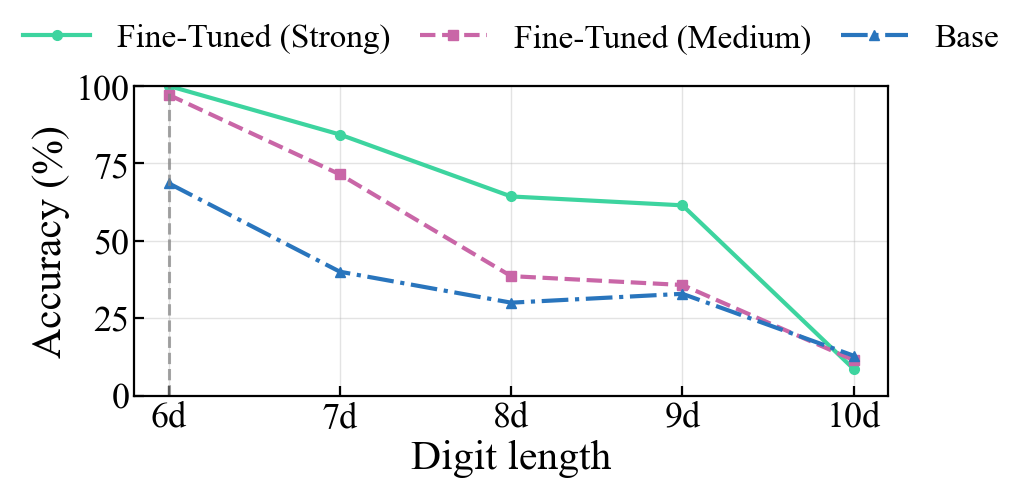}
        \caption{ \textbf{ICL accuracy on six- to ten-digit addition for three model
variants with different intrinsic
ICL capability.} Performance follows
        intrinsic ICL capability near the pretraining distribution, but as input
        digit length increases and the shift grows, performance gaps shrink and models
        converge to similarly low accuracy, indicating the Lipschitz constant dominates
        the performance bound. }
        \label{fig:pretrain_error}
    \end{subfigure}
    \hfill
    \begin{subfigure}
        [t]{0.32\textwidth}
        \centering
        \includegraphics[width=\textwidth]{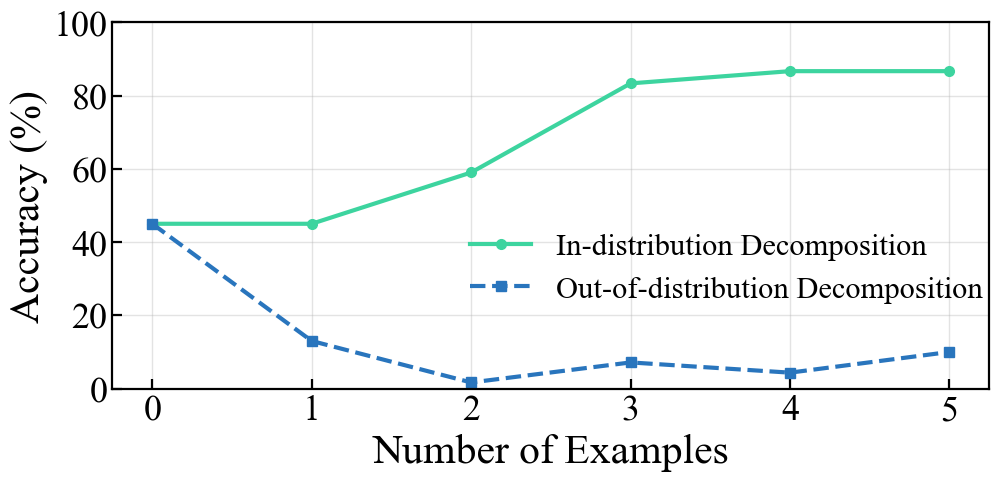}
        \caption{ \textbf{Effect of task decomposition quality in CoT prompting.}
        In-distribution decomposition, where each subtask aligns with computations
        learned during pretraining, substantially improves performance. In contrast,
        out-of-distribution decomposition performs worse than vanilla ICL
        without CoT and degrades further as the number of demonstrations
        increases. }
        \label{fig:chain}
    \end{subfigure}
    \hfill
    \begin{subfigure}
        [t]{0.32\textwidth}
        \centering
        \includegraphics[width=\textwidth]{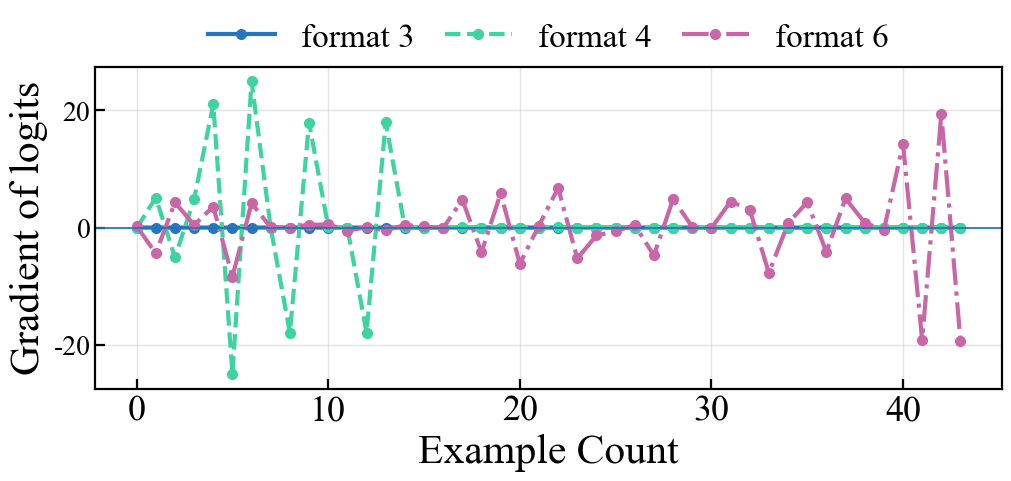}
        \caption{ \textbf{Average gradient magnitude of output logits with
        respect to instruction tokens under different prompt formats.} With
        correct or consistently correct/incorrect instructions (prompts~3 and~5),
        sensitivity is low or decays as demonstrations increase, whereas with
        inconsistent incorrect instructions (prompt~6), sensitivity persists and
        even increases with more demonstrations. }
        \label{fig:instruction}
    \end{subfigure}
    \caption{ Empirical validation of (a) \cref{thm:lipschitz}, (b)
    \cref{thm:cot_bound}, and (c) \cref{thm:decay_1_5,thm:decay_6}, respectively.}
    \label{fig:experiment}
    \vspace{-0.5em}
\end{figure*}

\vspace{-0.5em}
\section{Experiments}
\label{sec:experiments}

\subsection{Quantifying Demonstration Effectiveness under Prompt Shift}
\label{subsec:experiment_general} In \cref{thm:lipschitz}, we show that ICL
performance on unseen prompts is governed by a Lipschitz constant, which depends
on the pretrained model’s intrinsic ICL capability and the quality of the selected
demonstrations. Accordingly, we design two experiments to validate the effects of
these two factors on ICL performance.

\vspace{-0.3em}
\paragraph{Impact of Intrinsic ICL Capability and Prompt Shift.}
To study how intrinsic ICL capability affects performance, we prepare three variants
of \texttt{Llama3.2-3B} with different levels of intrinsic ICL capability. \textbf{(Setup)}
We consider the number addition task. (i) \texttt{Llama3.2-3B-Fine-Tuned (Strong)}:
We randomly generate 20,000 six-digit addition sequences and fine-tune the model
on them. (ii) \texttt{Llama3.2-3B-Fine-Tuned (Medium)}: We concatenate all
generated sequences into a single large prompt and fine-tune the model on this
single prompt. (iii) \texttt{Llama3.2-3B-Base (Weak)}: We use the base model without
any additional training. We evaluate them on six- to ten-digit addition to examine
how intrinsic ICL capability affects ICL performance on 350 unseen prompts.
\textbf{(Results)} Test accuracy consistently follows the ordering of intrinsic ICL
capability, corroborating our theoretical prediction. Notably, as the inference
task moves further away from the pretraining distribution, i.e., with increasing
digit length, the performance gap narrows. This is consistent with the regime where
the Lipschitz constant dominates the bound: when the task shift is sufficiently
large, even strong intrinsic ICL capability cannot prevent degraded ICL performance.
See Appendix~\ref{app:pretrain_error} for detailed setup.

\vspace{-0.5em}
\paragraph{Impact of Demonstration Choice.}

We design two tasks for investigating the effect of demonstration choice.
\textbf{(Task 1. Sport Identification)} Given a country, the task is to output
its associated sport label. We construct two types of demonstrations: \textit{identifying},
where the sport is a distinctive national or traditional sport that strongly
pins down a unique rule, and \textit{ambiguous}, where the sport is related to the
country but can be explained by multiple plausible latent rules, providing
weaker evidence about which rule the prompt instantiates. \textbf{(Task 2. Person
Identification)} Given a company, the task is to output a canonical person name associated
with that company. Similarly, we construct \textit{identifying} demonstrations, where
the person is an unambiguous representative that uniquely identifies the company,
and \textit{ambiguous} demonstrations, where the person is associated with the
company but does not exclusively identify it. For both tasks, we sample
$k \in \{2, 4, 8\}$ demonstrations from the corresponding pool and evaluate on a
fixed test set. \textbf{(Results)} As shown in Tables~\ref{tab:demo_choice_sports}--\ref{tab:demo_choice_people},
the identifying pool consistently yields higher accuracy than the ambiguous pool
across both tasks and models (\texttt{Qwen-235B-A22B} and \texttt{Qwen-30B-A3B}),
and the performance gap generally widens as the number of demonstrations increases.
This confirms that more informative demonstrations lead to stronger task
identification and better in-context generalization, consistent with our theoretical
predictions. Additional dataset details and full evaluation results are provided
in Appendix~\ref{app:demo}.
\begin{table}[htbp]
    \centering
    \caption{\textbf{Test accuracy ($\uparrow$) of ICL under ambiguous and
    identifying demonstrations ($k=8$).} Identifying demonstrations consistently
    yield higher accuracy across both tasks and models, indicating more precise
    task identification and better in-context generalization.}
    \label{tab:demo_choice_acc} \resizebox{0.72\linewidth}{!}{
    \begin{tabular}{llcc}
        \toprule \textbf{Task}           & \textbf{Demonstration} & \textbf{Qwen-30B-A3B}                   & \textbf{Qwen-235B-A22B}                 \\
        \midrule \multirow{2}{*}{Sport}  & Ambiguous              & $0.164 \ensuremath{\pm}0.039$           & $0.202 \ensuremath{\pm}0.037$           \\
                                         & Identifying            & $\textbf{0.559 \ensuremath{\pm} 0.074}$ & $\textbf{0.602 \ensuremath{\pm} 0.043}$ \\
        \midrule \multirow{2}{*}{People} & Ambiguous              & $0.820 \ensuremath{\pm}0.027$           & $0.648 \ensuremath{\pm}0.008$           \\
                                         & Identifying            & $\textbf{0.990 \ensuremath{\pm} 0.001}$ & $\textbf{0.845 \ensuremath{\pm} 0.044}$ \\
        \bottomrule
    \end{tabular}
    }
    \vspace{-0.5em}
\end{table}

\vspace{-0.5em}
\subsection{Effect of Task Decomposition Quality in CoT}
\label{subsec:experiment_cot}

In \cref{thm:cot_bound}, we show that the CoT performance depends on the subtask
selection, or task decomposition. If we decompose the task into easier subtasks,
CoT improves ICL performance; otherwise, it could be worse. We design a
synthetic experiment to validate the implication.

\textbf{(Setup)} We perform two-stage training for \texttt{Llama3.2-3B} using
LoRA. We first train the model on 20K six-digit addition samples, while the
pretrained model already performs one- to five-digit addition and simple multiplication
reliably. Our goal is to enable seven-digit addition via CoT prompting. To this end,
we consider two decomposition strategies: (i) \textit{In-Distribution
Decomposition}, which decomposes the task into sub-computations the model has already
mastered, namely six-digit addition, a one-digit addition, and a simple
multiplication by 10; (ii) \textit{Out-of-Distribution Decomposition}, which decomposes
the task into intermediate steps that still require seven-digit addition,
effectively preserving the original difficulty and introducing sub-tasks the
model cannot reliably solve in isolation. We construct a second training set of 500K
samples using the in-distribution decomposition strategy, and evaluate the model
under both decomposition strategies, with prompts containing zero to five demonstrations.
\textbf{(Results)} Results are shown in \cref{fig:chain}. In-distribution
decomposition substantially improves performance, whereas out-of-distribution decomposition
performs worse than vanilla ICL without CoT and degrades further as the number of
demonstrations increases. This behavior supports our result in
\cref{thm:cot_bound}. Prompt templates and additional experimental details are provided
in Appendix~\ref{app:cot_effect}.

\vspace{-0.5em}
\subsection{Prompt Sensitivity vs. Demonstration Count}
\label{subsec:experiment_instruction} In \cref{sec:instruction},
\cref{thm:decay_1_5,thm:decay_6} show that prompt templates exhibit different
synergy with demonstration count in ICL. For prompt formats 1–5, as the number
of demonstrations increases, ICL performance becomes less sensitive to the instruction.
In contrast, for prompt~6, where all instructions are incorrect and different, prompt
sensitivity persists even with many demonstrations. We design the following
experiment to validate this insight.

\textbf{(Setup)} We use a pretrained \texttt{Qwen3-235B-A22B} model and consider
three prompt formats: prompt~3 and prompt~5 from \cref{thm:decay_1_5}, and
prompt~6 from \cref{thm:decay_6}. We focus on the three-digit addition task. For
prompt~3, we use the correct operator $+$ in all demonstrations and test queries.
For prompt~5, we replace all $+$ symbols with $\times$, while keeping the
demonstration outputs as sums, requiring the model to infer the true task from examples.
For prompt~6, we randomly replace $+$ with $\times$, $/$, or $-$ for all demonstrations,
introducing fully incorrect and inconsistent instructions. \textbf{(Results)} We
evaluate the model under the three prompt formats with 1 to 50 demonstrations. We
visualize the gradient of the output logits with respect to the instruction tokens
in \cref{fig:instruction}. For prompt~3, the gradient remains near zero throughout,
indicating low instruction sensitivity. For prompt~5, instruction sensitivity is
initially high, reflecting a bias toward multiplication, and gradually decreases
as the model infers the correct task from demonstrations. In contrast, for
prompt~6, the gradient magnitude increases with more demonstrations, indicating
persistent and even amplified instruction sensitivity. These observations are consistent
with our theoretical analysis. Additional settings and results are provided in
Appendix~\ref{app:experiment_instruction_variation}.

\vspace{-0.5em}
\section{Conclusion}
\label{sec:conclusion}

In this work, we characterize in-context-learning (ICL) performance in terms of
the intrinsic capability, the distance between test and pretraining prompts, and
the task-identification power of demonstrations. Specifically, we show that
effective demonstrations correspond to low-Lipschitz paths that reliably
identify the underlying task, whereas poorly chosen demonstrations or format variations
can destabilize ICL by inducing perturbations. We also establish bounds showing
that, under task-consistent prompts, the influence of instructions vanishes as
the number of demonstrations increases, whereas inconsistent demonstrations or format
variations prevent task identification. We further extend the analysis to CoT by
treating multi-step reasoning as a sequence of task-identification steps and
deriving conditions under which this decomposition remains stable. Our paper develops
a theoretical framework to analyze how the relationship between prompts and pre-training
data affects the quality of in-context learning. This lays a theoretical
foundation for further research on prompting, including topics such as hallucination
and security.

%% file: list_notation.tex
\label{app:notation} We summarize the main notations used throughout the paper
and appendix for reference.

\begin{itemize}[label=\raisebox{0.25ex}{\tiny$\bullet$}, leftmargin=1.2em]
    \item $\gV$: finite vocabulary.

    \item $\gV^{*}$: set of all sequences of bounded length over $\gV$.

    \item $\gM$: class of Transformer models.

    \item $M \in \gM$: pretrained Transformer model.

    \item $M^{\star}$: target (oracle) model.

    \item $z$: a test prompt.

    \item $z^{(k)}$: prompt at CoT step $k$.

    \item $z^{\star (k)}$: oracle rollout prompt at step $k$.

    \item $\dicl$: ICL prompt domain. Throughout the paper, we assume that
        $\dicl$ is compact.

    \item $\dpre$: pretraining prompt domain.

    \item $D_{0}^{(k)}$: pretraining prompt domain for step-$k$ CoT subtask.

    \item $(x,y)$: input-output pair.

    \item $x_{\mathrm{query}}$: query input in the prompt.

    \item $y_{\mathrm{query}}$: model output for the query.

    \item $\ell_{\mathrm{ICL}}(z)$: ICL loss, defined as
        \[
            \ell_{\mathrm{ICL}}(z) := |M(z) - M^{\star}(z)|.
        \]

    \item $\gamma(t)$: straight-line path between prompts, $t\in[0,1]$.

    \item $\ell_{\mathrm{ICL}}^{\gamma}(t)$: loss restricted to path $\gamma$.

    \item $\lipschitz_{z,z'}$: Lipschitz constant of $\ell_{\mathrm{ICL}}$ on a
        compact set containing the path between $z$ and $z'$.

    \item $\lipschitz_{\gamma}$: pathwise Lipschitz constant of path $\gamma$.


    \item $J_{\gamma}$: interval induced by $D$ along path $\gamma$.

    \item $E_{\gamma}$: interval induced by $D_{0}$ along path $\gamma$.

    \item $\mes(\cdot)$: Lebesgue measure.

    \item $\gT$: finite latent task set.

    \item $t^{\star}$: true latent task.

    \item $\pi(t)$: prior over tasks.

    \item $p(t \mid z)$: posterior distribution over tasks.

    \item $p(y_{\mathrm{query}}\mid z)$: posterior predictive distribution.

    \item $f_{t}$: task-specific input-output mapping.

    \item $i$: instruction (or instruction sequence).

    \item $D_{i}$: extended instruction domain.

    \item $D_{i0}$: inner instruction domain (formats 1--5).

    \item $B_{r}, B_{R}$: Euclidean balls with radius $r$ and $R$.

    \item $L_{i}$: instruction-Lipschitz constant.

    \item $\nabla_{i}$: gradient with respect to instruction variables.

    \item $K$: number of CoT steps.

    \item $L^{\star}_{k}$: Lipschitz constant of the oracle model at step $k$.


    \item $B_{n}f$: degree-$n$ Bernstein polynomial approximation of $f$.

    \item $C_{\gamma}$: Remez amplification constant along path $\gamma$.

    \item $C_{V}$: Uniform constant only depending on the model structure.

    \item $\vz_{\mathrm{pad}}$: padded prompt with special tokens.

    \item $\vz_{\mathrm{orig}}$: original prompt without padding.
\end{itemize}

%% file: related_work.tex
\paragraph{Existing theories rely on strong assumptions.}
A substantial body of theoretical work seeks to understand the mechanisms
underlying ICL by analyzing how transformers implement learning procedures over
context. At the mechanistic level, studies identify specific architectural
patterns such as induction heads and function vectors—that enable example reuse
and pattern matching across the context \citep{Olsson2022, Elhage2021, todd2024}.
Complementary algorithmic analyses show that, under linear or kernelized assumptions,
ICL can be interpreted as implicitly gradient descent or kernel regression in
activation space \citep{Akyurek2023, vonOswald2023, Ahn2023, Mahankali2023}, or
as a form of task selection induced by the prior over pretraining tasks \citep{Bai2023transformers}.
While these mechanistic and algorithmic perspectives have significantly advanced
our understanding of how ICL is computed, they typically rely on restrictive architectural
assumptions, such as linear attention, single-layer or attention-only transformers,
or single-head attention—and on idealized data distributions such as random
generalized linear regression, which abstract away properties that are critical in
practice \citep{vonOswald2023, Mahankali2023, JMLR:v25:23-1042}. A complementary
line of work adopts a Bayesian or statistical inference perspective,
interpreting ICL as approximate posterior inference over latent tasks
conditioned on in-context demonstrations \citep{Xie2022, Wang2023(b), Panwar2024}.
This framing provides a principled account of how generalization improves with
more demonstrations, and has been extended to study the effects of noise,
misspecification, and pretraining bias on ICL behavior \citep{Gong2025towards}. However,
existing analyses within this framework uniformly assume a fixed and canonical
context format, treating demonstrations as drawn from a coherent, well-specified
data-generating process. As a result, they do not characterize how ICL behaves
when the format of the context itself varies—a limitation that our work directly
addresses.

\paragraph{Practical factors affect ICL generalization.}
Beyond mechanistic and statistical accounts, a large body of empirical work
investigates how practical design choices in prompt construction affect ICL
performance. A first line of work focuses on the role of demonstrations themselves:
studies show that the number of demonstrations, their selection strategy, and
the correctness of their labels each substantially influence downstream
performance \citep{Agarwal2024, Liu2022, su2023selective}. Notably, \citet{Min2022}
find that label correctness in demonstrations matters less than expected, suggesting
that ICL generalizes by retrieving a latent task structure rather than learning directly
from individual examples. A second line of work examines chain-of-thought
prompting, showing that including intermediate reasoning steps in demonstrations
can dramatically improve performance on complex tasks \citep{Wei2022, kojima2022, nye2021}.
Subsequent theoretical work explains CoT's benefit in terms of expressivity: CoT
enables deeper computation without additional layers \citep{Feng2023}, expands
the model's recognition capacity for complex problem classes \citep{Merrill2024},
or supports more serial computation than standard forward passes allow \citep{Li2024}.
However, these expressivity-based accounts do not characterize how CoT-structured
demonstrations affect ICL generalization as the number of reasoning steps grows,
which is a gap our work addresses through a multi-step error propagation
analysis. A third line of work studies sensitivity to prompt format and template
design, finding that surface-level variations in instruction wording, demonstration
ordering, and label presentation can lead to large performance swings \citep{ye2024investigating, Zhao2021, Webson2022}.
Despite the breadth of these empirical findings, they remain largely descriptive:
existing work does not provide a theoretical criterion for when format variation
preserves successful generalization and when it fundamentally disrupts
consistent task inference. Our work provides such a criterion by formally distinguishing
prompt formats under which the posterior predictive converges exponentially with
more demonstrations from those under which convergence necessarily fails.

\paragraph{Mechanistic explanations of ICL.}
A complementary line of work try to explain ICL at the level of internal
representations and computational mechanisms, examining how specific architectural
components give rise to in-context behavior. At the representational level, studies
identify induction heads and function vectors as key circuits that enable pattern
matching and example reuse across the context window \citep{Olsson2022, Elhage2021, todd2024},
and related work shows that multi-head softmax attention can implement regression
behaviors in simplified settings \citep{He2025incontext, Li2025provable}. At the
algorithmic level, several studies argue that transformers implicitly perform
gradient descent or kernel regression over in-context examples \citep{Akyurek2023, vonOswald2023, Ahn2023},
and related perspectives characterize ICL as algorithm selection or prior-induced
regularization across tasks \citep{Bai2023transformers, Lu2025transformer}. A
separate line of work examines whether standard pretraining objectives can
produce the representational capacity assumed by these accounts, establishing conditions
under which pretrained transformers exhibit ICL-like behavior \citep{kimandsuzuki2024transformers, kim2024transformers}.
While these mechanistic and algorithmic explanations have significantly advanced
our understanding of the computational substrate of ICL, they characterize how
ICL is implemented within a fixed architecture and format, but do not address when
and why ICL generalizes across varying prompt formats, nor how practical design
choices such as demonstration count, CoT structure, or instruction variation
affect the quality of task inference. Our work is therefore complementary to this
line, operating at the level of generalization theory rather than mechanistic
realization.

%% file: appendix_prelim.tex
\label{app:lipschitz}
In this appendix, we provide a complete proof of the formal Lipschitz generalization
bound stated in \cref{sec:general}.

\subsection{Preliminaries}

We first introduce Bernstein polynomials here, which we will use to construct polynomial
surrogates for the path-restricted loss. Detailed properties of Bernstein
polynomials can be found in~\citet{Lorentz1986}.

\noindent
For a function $f:[0,1]\to\mathbb{R}$ and an integer $n\ge 1$, its degree-$n$
Bernstein polynomial is defined by
\[
    B_{n}f(t) = \sum_{k=0}^{n}f\!\left(\frac{k}{n}\right) \binom{n}{k}t^{k}(1-t)^{n-k}
    , \qquad t\in[0,1].
\]
\noindent
The function $B_{n}f$ is a polynomial of degree at most $n$. Moreover, if $f$ is
Lipschitz continuous with constant $l$ on $[0,1]$, then (see Lemma~\ref{thm:berstein})
\[
    \sup_{t\in[0,1]}|B_{n}f(t)-f(t)| \le \frac{l}{2\sqrt{n}}.
\]

\noindent
The proof relies on two standard tools from approximation theory: the Bernstein polynomial
approximation theorem and the Remez inequality. For completeness, we briefly recall
their statements in the form needed here.

\begin{lemma}[Remez Inequality for Polynomials~\citep{remez1936}]
    \label{thm:remez_app} If $J\in\sR$ is a finite interval, and $E \subset J$
    is an an arbitrary measurable set, then for any polynomial $p$ of degree $n$,
    \begin{equation}
        \max_{x \in J}\abs{p(x)}\leq \pt{\frac{4 \mes (J)}{\mes (E)}}^{n}\sup_{x\in
        E}\abs{p(x)}.
    \end{equation}
\end{lemma}

\begin{lemma}[Bernstein Approximation Theorem, Lipschitz Case~\citep{Lorentz1986}]
    \label{thm:berstein_app} Let $f: [0,1] \to \sR$ be a Lipschitz continuous function
    with Lipschitz constant $\lipschitz$. Let $B_{n}f$ denote the degree-$n$
    Bernstein polynomail (see Appendix~\ref{app:lipschitz} for definition)
    approximation of $f$. Then, for all $n\geq 1$,
    \begin{equation}
        \sup_{x\in[0,1]}\abs{B_n f(x) - f(x)}\leq \frac{\lipschitz}{2\sqrt{n}}.
    \end{equation}
\end{lemma}

\noindent
We then introduce the assumptions required for the proof.

\begin{assumption}
    [Local Lipschitzness along paths] \label{ass:path_lip} For every pair $(\vz,\vz
    ')$ considered, there exists a compact set
    $\mathcal{K}_{\vz,\vz'}\subset\gV^{*}$ such that
    $\gamma([0,1])\subset\mathcal{K}_{\vz,\vz'}$ and the loss $\error$ is
    Lipschitz on $\mathcal{K}_{\vz,\vz'}$ with constant $L_{\vz,\vz'}<\infty$:
    \[
        |\error(\va)-\error(\vb)| \le L_{\vz,\vz'}\|\va-\vb\|, \qquad \forall \va
        ,\vb\in\mathcal{K}_{\vz,\vz'}.
    \]
\end{assumption}

\noindent
This assumption states that the loss varies in a controlled manner along any
straight-line interpolation between a test prompt and a reference pretraining
prompt. We only require Lipschitz continuity on a compact neighborhood of the path,
rather than globally over $\gV^{*}$. We perform the analysis in the continuous embedding
representation space induced by the tokenizer and embedding layer. Although the
token domain $\gV^{*}$ is discrete, its embedding representation lies in a continuous
vector space, where the interpolation paths used in our proof are defined.

Recall that $\gamma$ refers to the path between test prompt $\vz$ and
pretraining prompt $\vz'$, $\E_{\gamma}$ refers to interval induced by $D_{0}$ along
path $\gamma$.
\begin{assumption}
    [Positive pretraining coverage along a connecting path] \label{ass:positive_measure}
    There exist a point $\vz' \in \dpre$ and a continuous path $\gamma : [0,1] \to
    \dicl$ connecting $\vz$ and $\vz'$ (i.e., $\gamma(0)=\vz$ and
    $\gamma(1)=\vz'$) such that $\mes(E_{\gamma}) > 0.$
\end{assumption}

\noindent
This assumption ensures that the pretraining domain occupies a non-negligible portion
of the interpolation path between $\vz'$ and $\vz$. Equivalently, the model must
have seen a positive-measure subset of prompts along the direction connecting the
test prompt to the pretraining domain.

\subsection{Proof of \cref{thm:lipschitz}}
We first present the formal statement of the Lipschitz generalization bound,
which is a more detailed version of \cref{thm:lipschitz} in the main text.
\begin{theorem}[ICL Generalization Bound, formal version of \cref{thm:lipschitz}]
    \label{thm:lipschitz_formal} Let $\vz\notin\dpre$ and suppose Assumptions~\ref{ass:path_lip}
    and~\ref{ass:positive_measure} hold. Fix any $\vz'\in\dpre$ such that
    $\mes(E_{\gamma})>0$ for the straight-line path $\gamma(t)=\vz'+t(\vz-\vz')$.
    Let $l_{\gamma}:=L_{\vz,\vz'}\|\vz-\vz'\|$ and
    \[
        A_{\gamma}:= \frac{4\,\mes(J_{\gamma})}{\mes(E_{\gamma})}> 1,
    \]
    where $\mes(J_{\gamma})$ is a constant.

    \noindent
    Then for any integer $n\ge 1$,
    \begin{equation}
        \label{eq:thm_lip_any_n}\error(\vz) \le A_{\gamma}^{n}\sup_{\tilde{\vz}\in\dpre}
        \error(\tilde{\vz}) +\Bigl(1+A_{\gamma}^{n}\Bigr)\frac{l_{\gamma}}{2\sqrt{n}}
        .
    \end{equation}
    In particular, choosing $n=\lceil l_{\gamma}^{2}\rceil$ yields
    \begin{equation}
        \label{eq:thm_lip_chosen_n}\error(\vz) \le A_{\gamma}^{\lceil l_\gamma^2\rceil}
        \sup_{\tilde{\vz}\in\dpre}\error(\tilde{\vz}) +\frac{1+A_{\gamma}^{\lceil
        l_\gamma^2\rceil}}{2}.
    \end{equation}
\end{theorem}

\noindent
We now proceed with the proof, following the same four conceptual steps as outlined
in the proof sketch.

\begin{proof}
    \textbf{Step 1: Reduce the prompt space to a finite interval via a path.}
    Fix $\vz\notin\dpre$ and choose $\vz'\in\dpre$ such that
    $\mes(E_{\gamma})>0$ for the path $\gamma(t)=\vz'+t(\vz-\vz')$, $t\in[0,1]$.
    Define the path-restricted loss
    \[
        \error^{\gamma}(t):=\error(\gamma(t)),\qquad t\in[0,1].
    \]
    By Assumption~\ref{ass:path_lip}, $\error$ is Lipschitz on a compact set containing
    $\gamma([0,1])$, hence for all $t_{1},t_{2}\in[0,1]$,
    \[
        |\error^{\gamma}(t_{1})-\error^{\gamma}(t_{2})| \le L_{\vz,\vz'}\|\gamma(
        t_{1})-\gamma(t_{2})\| = L_{\vz,\vz'}\|\vz-\vz'\||t_{1}-t_{2}|.
    \]
    Denote $l_{\gamma}:=L_{\vz,\vz'}\|\vz-\vz'\|$.

    We further denote the intervals induced along the path by
    \[
        J_{\gamma}:= [0,1], \qquad E_{\gamma}:= \{t\in J_{\gamma}:\gamma(t)\in\dpre
        \}.
    \]
    By construction and Assumption~\ref{ass:positive_measure},
    $\mes(E_{\gamma})>0$.

    \medskip
    \textbf{Step 2: Convert a Lipschitz function into a polynomial via Bernstein
    approximation.} Applying Lemma~\ref{thm:berstein_app} to the Lipschitz function
    $\error^{\gamma}$ gives
    \[
        \sup_{t\in[0,1]}|B_{n}\error^{\gamma}(t)-\error^{\gamma}(t)| \le \frac{l_{\gamma}}{2\sqrt{n}}
        .
    \]
    Note that $B_{n}\error^{\gamma}$ is a polynomial of degree at most $n$.

    \medskip
    \textbf{Step 3: Apply Remez--Chebyshev to the polynomial surrogate.}
    Applying Lemma~\ref{thm:remez_app} to $p=B_{n}\error^{\gamma}$ on the
    interval $J_{\gamma}$ with subset $E_{\gamma}$ yields
    \[
        \max_{t\in J_\gamma}|B_{n}\error^{\gamma}(t)| \le \Bigl(\frac{4\,\mes(J_{\gamma})}{\mes(E_{\gamma})}
        \Bigr)^{n}\sup_{t\in E_\gamma}|B_{n}\error^{\gamma}(t)|.
    \]
    Denote $A_{\gamma}:=\frac{4\,\mes(J_{\gamma})}{\mes(E_{\gamma})}$.

    \medskip
    \textbf{Step 4: Transfer the bound back to $\error$ and choose $n$.} For any
    $t\in J_{\gamma}$,
    \[
        \error^{\gamma}(t) \le B_{n}\error^{\gamma}(t)+|B_{n}\error^{\gamma}(t)-\error
        ^{\gamma}(t)| \le B_{n}\error^{\gamma}(t)+\frac{l_{\gamma}}{2\sqrt{n}}.
    \]
    Taking maxima over $t\in J_{\gamma}$ and using the Remez bound,
    \[
        \max_{t\in J_\gamma}\error^{\gamma}(t) \le A_{\gamma}^{n}\sup_{t\in E_\gamma}
        B_{n}\error^{\gamma}(t)+\frac{l_{\gamma}}{2\sqrt{n}}.
    \]
    Moreover, for $t\in E_{\gamma}$,
    \[
        B_{n}\error^{\gamma}(t) \le \error^{\gamma}(t)+\frac{l_{\gamma}}{2\sqrt{n}}
        .
    \]
    Combining the last two displays yields
    \[
        \max_{t\in J_\gamma}\error^{\gamma}(t) \le A_{\gamma}^{n}\sup_{t\in E_\gamma}
        \error^{\gamma}(t) +\Bigl(1+A_{\gamma}^{n}\Bigr)\frac{l_{\gamma}}{2\sqrt{n}}
        .
    \]
    Since $\gamma(t)\in\dpre$ for $t\in E_{\gamma}$,
    \[
        \sup_{t\in E_\gamma}\error^{\gamma}(t)\le \sup_{\tilde{\vz}\in\dpre}\error
        (\tilde{\vz}).
    \]
    Finally, $\error(\vz)=\error^{\gamma}(1)\le \max_{t\in J_\gamma}\error^{\gamma}
    (t)$, which proves~\eqref{eq:thm_lip_any_n}. The specialization~\eqref{eq:thm_lip_chosen_n}
    follows by choosing $n=\lceil l_{\gamma}^{2}\rceil$ and noting $l_{\gamma}/\sqrt{n}
    \le 1$.
\end{proof}

%% file: appendix_padding.tex
\label{app:padding}

\subsection{Preliminaries}
In this appendix, we justify the general padding assumption mentioned in \cref{sec:result_lipschitz},
which is used throughout the analysis in \cref{thm:lipschitz} and
\cref{thm:cot_bound}. The purpose of padding is to align the effective input domain
at inference with that of pretraining, without altering the underlying function
class implemented by the model.

\paragraph{Definition of Padding}
In practice, prompts may have different lengths. To enable a unified functional analysis,
we embed all prompts into a fixed-length space via padding.
\noindent
For any prompt $\vz$, we define a padded version $\vz_{\mathrm{pad}}$ obtained
by inserting a special token $\langle\mathrm{pad}\rangle$ at arbitrary positions,
subject to the following constraints: the original non-padding tokens remain
unchanged; their relative order is preserved; and only additional
$\langle\mathrm{pad}\rangle$ tokens are inserted. Without padding, prompts of
varying lengths cannot be treated as elements of a common ambient space, precluding
a unified functional analysis.

\subsection{Stability of the Transformer Under Padding}
\label{app:padding_general} Let $y$ and $y'$ denote the attention outputs at a fixed
query position before and after adding padding tokens, respectively. Let $\vz$ be
the original prompt and $\vz_{\mathrm{pad}}$ be the padded prompt obtained by inserting
$m$ padding tokens without modifying any non-padding token. We assume that the
embedding of $\langle\mathrm{pad}\rangle$ is the zero vector, so that the input representations
satisfy $\vz_{\mathrm{pad}}^{(0)}= \vz^{(0)}$ at positions corresponding to original
tokens.
\noindent
We show that, under mild conditions, the attention output is stable under
padding: $\|y' - y\|_{2}$ can be made arbitrarily small by controlling (i) the number
$m$ of padding tokens and (ii) their relative attention logits.

\begin{lemma}[Attention Stability Under Padding]
    \label{lmm:pad_output_general} Let $y$ and $y'$ denote the attention outputs
    before and after adding $m$ padding tokens. Let $x_{i}$ denote the attention
    score assigned to token $i$, $x_{p}$ denote the attention score assigned to pad
    token $p$, respectively. Let $\vz$ be the original prompt and $\vz_{\mathrm{pad}}$ be the padded prompt obtained by inserting
$m$ padding tokens. Then there exists a uniform constant $C_{V}$ such that
    \[
        \|y' - y\|_{2}\le 2C_{V}\cdot \frac{\vz_{\mathrm{pad}}}{\vz_{\mathrm{orig}}+\vz_{\mathrm{pad}}}
        \le 2C_{V}\cdot \frac{m e^{x_{\mathrm{pad,max}}}}{\vz_{\mathrm{orig}}}.
    \]
    In particular, for any tolerance $\varepsilon \in (0, 2C_{V})$, the condition
    \[
        m \le \frac{\varepsilon}{2C_{V}-\varepsilon}e^{\Delta}
    \]
    is sufficient to ensure $\|y' - y\|_{2}\le \varepsilon$, where
    $j \in \mathcal{I}$ indexes a dominant informative token, and $\Delta := x_{j}
    - x_{\mathrm{pad,max}}$ denotes the softmax gap between informative and padding
    tokens.
\end{lemma}

\begin{proof}
    We use the standard attention form at a fixed query position:
    \[
        y = \frac{1}{\vz_{\mathrm{orig}}}\sum_{i\in\mathcal{I}}e^{x_i}v_{i}, \qquad
        y' = \frac{1}{\vz_{\mathrm{orig}}+\vz_{\mathrm{pad}}}\left( \sum_{i\in\mathcal{I}}
        e^{x_i}v_{i}+ \sum_{p\in\mathcal{P}}e^{x_p}v_{p}\right),
    \]
    where
    \[
        \vz_{\mathrm{orig}}:= \sum_{i\in\mathcal{I}}e^{x_i}, \qquad \vz_{\mathrm{pad}}
        := \sum_{p\in\mathcal{P}}e^{x_p}.
    \]

    Subtracting $y$ from $y'$ gives
    \[
        y'-y = \frac{\sum_{p\in\mathcal{P}}e^{x_p}v_{p}}{\vz_{\mathrm{orig}}+\vz_{\mathrm{pad}}}
        - \frac{\vz_{\mathrm{pad}}}{\vz_{\mathrm{orig}}(\vz_{\mathrm{orig}}+\vz_{\mathrm{pad}})}
        \sum_{i\in\mathcal{I}}e^{x_i}v_{i}.
    \]

    Taking $\ell_{2}$ norms and using the triangle inequality,
    \[
        \|y'-y\|_{2}\le \frac{\bigl\|\sum_{p\in\mathcal{P}}e^{x_p}v_{p}\bigr\|_{2}}{\vz_{\mathrm{orig}}+\vz_{\mathrm{pad}}}
        + \frac{\vz_{\mathrm{pad}}\bigl\|\sum_{i\in\mathcal{I}}e^{x_i}v_{i}\bigr\|_{2}}{\vz_{\mathrm{orig}}(\vz_{\mathrm{orig}}+\vz_{\mathrm{pad}})}
        .
    \]

    Let $C_{V}$ be a uniform upper bound on $\|v_{i}\|_{2}$. Then
    \[
        \Bigl\|\sum_{p\in\mathcal{P}}e^{x_p}v_{p}\Bigr\|_{2}\le C_{V}\vz_{\mathrm{pad}}
        , \qquad \Bigl\|\sum_{i\in\mathcal{I}}e^{x_i}v_{i}\Bigr\|_{2}\le C_{V}\vz
        _{\mathrm{orig}}.
    \]

    Substituting yields
    \[
        \|y'-y\|_{2}\le 2C_{V}\frac{\vz_{\mathrm{pad}}}{\vz_{\mathrm{orig}}+\vz_{\mathrm{pad}}}
        .
    \]

    Let
    \[
        x_{\mathrm{pad,max}}:=\max_{p\in \mathcal{P}}x_{p}
    \]
    denote the maximum attention logit assigned to any padding token, we further
    obtain
    \[
        \|y' - y\|_{2}\le 2C_{V}\,\frac{m e^{x_{\mathrm{pad,max}}}}{\vz_{\mathrm{orig}}}
        ,
    \]
    which establishes the first two inequalities in the lemma statement.

    For the sufficient condition, solving
    \[
        2C_{V}\frac{\vz_{\mathrm{pad}}}{\vz_{\mathrm{orig}}+\vz_{\mathrm{pad}}}\le
        \varepsilon
    \]
    gives
    \[
        \vz_{\mathrm{pad}}\le \frac{\varepsilon}{2C_{V}-\varepsilon}\, \vz_{\mathrm{orig}}
        .
    \]

    If there exists $j\in\mathcal{I}$ such that
    $\Delta = x_{j}- x_{\mathrm{pad,max}}> 0$, then
    \[
        \vz_{\mathrm{pad}}\le m e^{x_{\mathrm{pad,max}}}= m e^{x_j-\Delta}, \qquad
        \vz_{\mathrm{orig}}\ge e^{x_j}.
    \]

    Hence
    \[
        \frac{\vz_{\mathrm{pad}}}{\vz_{\mathrm{orig}}}\le m e^{-\Delta}.
    \]

    Imposing
    \[
        m \le e^{\Delta}\frac{\varepsilon}{2C_{V}-\varepsilon}
    \]
    yields the stated sufficient condition.
\end{proof}

\noindent
Lemma~\ref{lmm:pad_output_general} shows that padding tokens act as structured distractors
whose only influence on the attention output is through their contribution to the
softmax normalization mass. The perturbation does not depend on the semantic
content of the padding tokens themselves, but is governed entirely by two quantitative
factors: (i) the number $m$ of padding tokens and (ii) their relative attention
logits compared to informative tokens.
\noindent
More precisely, the stability bound depends on the quantity $m e^{-\Delta}$, where
$\Delta$ denotes the softmax gap between a dominant informative token and the most
competitive padding token. Thus, stability is ensured either when the padding length
$m$ is controlled, or when padding tokens receive sufficiently small attention
scores (i.e., when $\Delta$ is large).

%% file: appendix_cot.tex
\label{app:cot_bound_proof}

In this appendix we provide a complete proof of the multi-step error propagation
bound underlying \cref{sec:cot}.
\subsection{Preliminaries}
We briefly recall the main objects appearing in the multi-step CoT analysis and
define new notations essential for the proof.
\begin{itemize}
    \item Let $(\mathcal{Z},\|\cdot\|)$ be the prompt space. For a fixed initial
        prompt $\vz^{(0)}$, define the model and oracle rollouts by
        \[
            \vz^{(k+1)}= f\!\big(\vz^{(k)}, \model(\vz^{(k)})\big), \qquad \vz^{*(k+1)}
            = f\!\big(\vz^{*(k)}, \bestmodel(\vz^{*(k)})\big),
        \]
        with $\vz^{*(0)}=\vz^{(0)}$, where $f(\vz, a)$ denotes the operation of
        appending answer $a$ to prompt $\vz$.

    \item The learned model is denoted by $\model:\mathcal{Z}\to\mathbb{R}$, and
        the reference (oracle) predictor is denoted by
        $\bestmodel:\mathcal{Z}\to\mathbb{R}$. We define the single-step prediction
        error as
        \[
            \error(\vz) := |\model(\vz)-\bestmodel(\vz)|.
        \]

    \item We measure the discrepancy between the model and oracle rollouts by
        \[
            \Delta_{k}:= \|\vz^{(k)}-\vz^{*(k)}\|, \qquad \Delta_{0}=0.
        \]
\end{itemize}

\noindent
We then introduce the assumptions and lemmas required for the proof.
\begin{assumption}
    \label{ass:well_defined_rollout} Both sequences are well-defined for $k=0,\dots
    ,K-2$ and remain in $\mathcal{Z}$.
\end{assumption}

\begin{assumption}
    [Per-step error propagation] \label{ass:propagation} For each step $k \in \{0
    , \dots, K-2\}$, there exist constants $\alpha_{k}, \beta_{k}\geq 0$ such
    that
    \[
        \Delta_{k+1}\leq \beta_{k}\,\error(\vz^{(k)}) + \alpha_{k}\,\Delta_{k}.
    \]
\end{assumption}
\noindent
This assumption captures how errors propagate through a single CoT step. The term
$\beta_{k}\,\varepsilon(\mathbf{z}^{(k)})$reflects direct error injection: even if
the model and oracle rollouts are currently at the same state ($\Delta_{k}= 0$),
a prediction error $\varepsilon(\mathbf{z}^{(k)})$ at step $k$ will cause the
two rollouts to diverge at step $k+1$, with $\beta_{k}$ measuring how sensitive the
update $f$ is to the predicted answer. The term $\alpha_{k}\,\Delta_{k}$
reflects error accumulation: any existing divergence $\Delta_{k}$ between the two
rollouts is carried forward and potentially amplified by $\alpha_{k}$, which
captures the sensitivity of $f$ to the current prompt state as well as how the
oracle predictor's smoothness interacts with the update.

\begin{assumption}
    \label{ass:lipschitz_oracle_predictor} For each step $k$, there exists a set
    $\mathcal{R}_{k}$ containing all rollout states at step $k$ such that
    $\bestmodel$ is $L_{k}^{\star}$-Lipschitz on $\mathcal{R}_{k}$:
    \[
        |\bestmodel(\vz)-\bestmodel(\vz')| \le L_{k}^{\star}\|\vz-\vz'\|.
    \]
\end{assumption}

\begin{lemma}
    \label{lem:perstep_gen} Fix a CoT step $k\in\{0,\dots,K-1\}$. Let
    $\error^{k}:\mathcal{Z}\to\mathbb{R}_{+}$ denote the step-$k$ prediction
    error
    \[
        \error^{k}(\vz):=|\model(\vz^{(k)})-\bestmodel(\vz^{(k)})|,
    \]
    Assume the single-step Lipschitz generalization theorem (Theorem~\ref{thm:lipschitz})
    holds when instantiated to this step-$k$ subtask, with pretraining prompt set
    $\dpre^{(k)}\subset\mathcal{Z}$. Then there exist constants $C_{k}\ge 1$ and
    an effective Lipschitz parameter $\lipschitz_{(k)}\ge 0$ such that for any
    step-$k$ rollout prompt $\vz^{(k)}\notin \dpre^{(k)}$,
    \[
        \error^{k}(\vz^{(k)}) \le C_{k}\exp\!\big(\lipschitz_{(k)}^{2}\big)\cdot
        \sup_{\vz'\in\dpre^{(k)}}\error^{k}(\vz').
    \]
\end{lemma}

\begin{proof}
    We apply Theorem~\ref{thm:lipschitz} to the step-$k$ subtask. Concretely, we
    view $\error^{k}(\cdot)$ as the single-step loss function in Theorem~\ref{thm:lipschitz},
    and take the associated pretraining domain to be $\dpre^{(k)}$. By assumption,
    the regularity and coverage conditions required by Theorem~\ref{thm:lipschitz}
    hold for this step-$k$ instantiation, hence the conclusion of Theorem~\ref{thm:lipschitz}
    yields constants $C_{k}\ge 1$ and $\lipschitz_{(k)}\ge 0$ such that
    \[
        \error^{k}(\vz) \;\le\; C_{k}\exp\!\big(\lipschitz_{(k)}^{2}\big)\cdot \sup
        _{\vz'\in\dpre^{(k)}}\error^{k}(\vz'), \qquad \forall \vz\notin \dpre^{(k)}
        .
    \]
    Evaluating this inequality at $\vz=\vz^{(k)}$ gives the desired result.
\end{proof}

\subsection{Proof of Theorem~\ref{thm:cot_bound_formal}}
We first present a more general version of the theorem.
\begin{theorem}[ICL Generalization Bound with CoT Prompting, formal version of \cref{thm:cot_bound}]
    \label{thm:cot_bound_formal} Under Assumptions~\ref{ass:well_defined_rollout}--\ref{ass:lipschitz_oracle_predictor},
    \begin{equation}
        \label{eq:cot_main_bound}\error(\vz^{(K-1)}) \le R_{K-1}+ L_{K-1}^{\star}
        \sum_{m=0}^{K-2}\left( \prod_{j=m+1}^{K-2}\alpha_{j}\right) \beta_{m}\error
        (\vz^{(m)}),
    \end{equation}
    where $R_{K-1}:=|\model(\vz^{(K-1)})-\bestmodel(\vz^{*(K-1)})|$.

    Substituting the per-step generalization bound yields
    \begin{equation}
        \label{eq:cot_main_bound_gen}\error(\vz^{(K-1)}) \le R_{K-1}+ L_{K-1}^{\star}
        \sum_{m=0}^{K-2}\left( \prod_{j=m+1}^{K-2}\alpha_{j}\right) \beta_{m}C_{m}
        e^{\lipschitz_{(m)}^2}\sup_{\vz'\in\dpre^{(m)}}\error(\vz').
    \end{equation}
\end{theorem}

We now proceed with the proof.
\begin{proof}
    Fix $k\in\{0,\dots,K-2\}$. By Assumption~\ref{ass:propagation},
    \[
        \Delta_{k+1}\leq \beta_{k}\,\error(\vz^{(k)}) + \alpha_{k}\,\Delta_{k}.
    \]

    Iterating from $\Delta_{0}=0$ yields
    \[
        \Delta_{K-1}\le \sum_{m=0}^{K-2}\left( \prod_{j=m+1}^{K-2}\alpha_{j}\right
        ) \beta_{m}\,\error(\vz^{(m)}).
    \]

    By the triangle inequality,
    \begin{align}
        \error(\vz^{(K-1)}) & = |\model(\vz^{(K-1)})-\bestmodel(\vz^{(K-1)})|                                                          \\
                            & \le |\model(\vz^{(K-1)})-\bestmodel(\vz^{*(K-1)})| + |\bestmodel(\vz^{*(K-1)})-\bestmodel(\vz^{(K-1)})|.
    \end{align}

    By Assumption~\ref{ass:lipschitz_oracle_predictor}, since $\vz^{(K-1)}$ and
    $\vz^{*(K-1)}$ are rollout states at step $K-1$ and hence belong to
    $\mathcal{R}_{K-1}$, the second term satisfies
    \[
        |\bestmodel(\vz^{*(K-1)})-\bestmodel(\vz^{(K-1)})| \le L_{K-1}^{\star}\Delta
        _{K-1}.
    \]

    Substituting the bound on $\Delta_{K-1}$ gives
    \[
        \error(\vz^{(K-1)}) \le R_{K-1}+ L_{K-1}^{\star}\sum_{m=0}^{K-2}\left( \prod
        _{j=m+1}^{K-2}\alpha_{j}\right ) \beta_{m}\,\error(\vz^{(m)}),
    \]
    which is precisely \eqref{eq:cot_main_bound}.

    Next, by Lemma~\ref{lem:perstep_gen}, for each $m$ we have
    \[
        \error(\vz^{(m)}) \le C_{m}e^{\lipschitz_{(m)}^{2}}\sup_{\vz'\in\dpre^{(m)}}
        \error(\vz').
    \]
    Substituting this inequality term-by-term into the above sum gives \eqref{eq:cot_main_bound_gen}.
    This completes the proof.
\end{proof}

\subsection{Models can ignore irrelevant steps in CoT}
\label{app:ignore}
To empirically verify that language models can effectively ignore irrelevant
steps in a CoT prompt, we examine whether this suppression capability is acquired
during pretraining rather than imposed by the prompting format.
\paragraph{Setup.}
We consider the arithmetic addition task. We pretrain a \texttt{Llama3.2-3B} model
on addition while injecting varying proportions of Shakespearean text into the pretraining
corpus as structured distractors, where \texttt{Shakespeare-less}, \texttt{Shakespeare-mid},
and \texttt{Shakespeare-more} denote mixing ratios of 0\%, 20\%, and 50\%
Shakespearean content, respectively. During training, the model is instructed to
perform the addition specified by the prompt (here we use the \emph{In-distribution
decomposition} CoT chain provided in Appendix~\ref{app:cot_effect}), while the surrounding
Shakespearean passages are entirely irrelevant to the target computation. If the
model can still learn the intended input--output mapping despite substantial task-irrelevant
content, it must have developed the capacity to selectively ignore such
distractions. We evaluate model accuracy on held-out addition prompts across different
distractor ratios.
\paragraph{Results.}
As shown in Figure~\ref{fig:shakespeare}, even with a substantial proportion of
Shakespearean text included during pretraining, the model achieves strong performance
on the addition task when prompted with CoT. While accuracy degrades slightly as
distractor content increases, the model remains capable of effectively suppressing
irrelevant information and focusing on the core arithmetic operation. This
supports our theoretical claim in \cref{thm:cot_bound} that language models can
learn to ignore irrelevant CoT steps, enabling robust generalization even under significant
distraction.

\begin{figure}[htbp]
    \centering
    \includegraphics[width=0.4\textwidth]{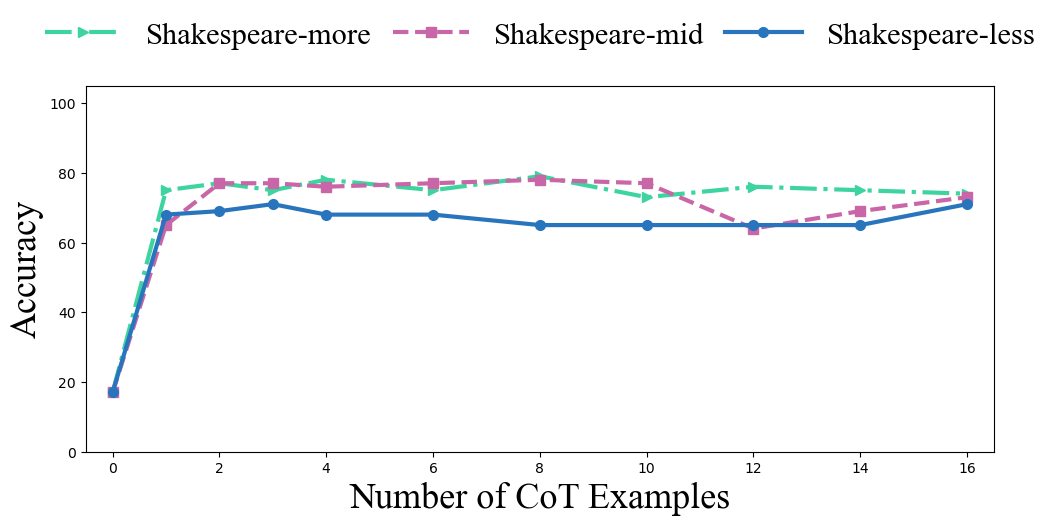}
    \includegraphics[width=0.4\textwidth]{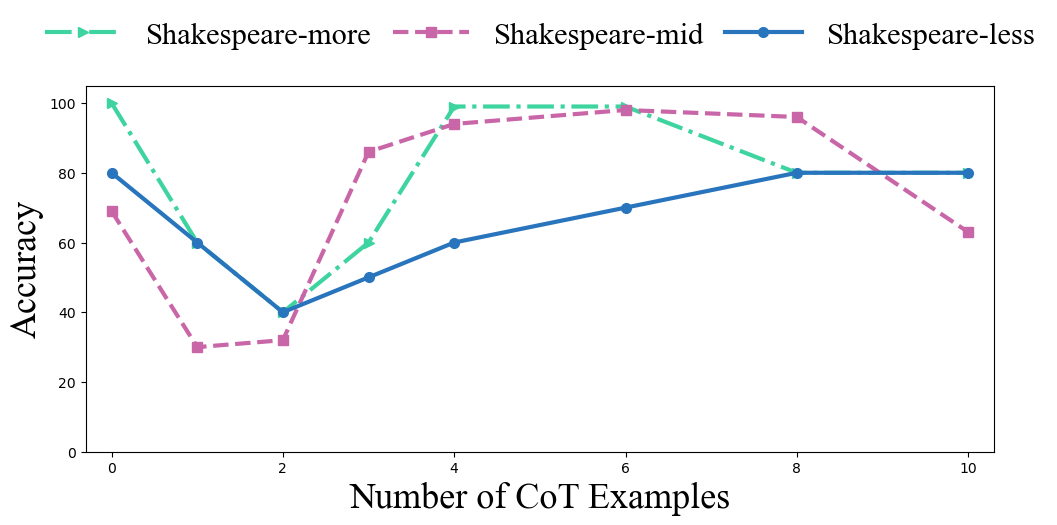}
    \caption{\textbf{Effect of structured distractors introduced during
    pretraining on CoT performance.} Shakespeare~less/mid/more denote progressively
    larger proportions of Shakespearean corpus mixed into the pretraining data.
    At evaluation time, the model is prompted to solve addition problems using CoT.
    Despite substantial irrelevant context, the model retains strong CoT performance
    at inference, indicating an ability to suppress task-irrelevant information.
    }
    \label{fig:shakespeare}
\end{figure}

%% file: appendix_instruction.tex
\label{app:instruction}
\subsection{Preliminaries}
We begin by expanding the posterior distribution over tasks using Bayes' rule. For
any task $t\in\gT$, conditioning on the full prompt
$\vz=(\vi,\{(x_{n},y_{n}) \}_{n=1}^{N},\xquery)$ yields
\begin{equation}
    \label{eq:posterior_bayes_full}p(t\mid \vz) = \frac{\pi(t)\,p(\vz\mid t)}{\sum_{s\in\gT}\pi(s)\,p(\vz\mid
    s)}.
\end{equation}
Under prompt formats 1--5, the likelihood $p(\vz\mid t)$ factorizes as
\begin{equation}
    \label{eq:likelihood_factorization}p(\vz\mid t) = p(\vi\mid t)\, p(\xquery\mid
    t)\, \prod_{n=1}^{N}p(x_{n},y_{n}\mid t),
\end{equation}
where the instruction, query input, and demonstration pairs are conditionally
independent given the latent task $t$.

Substituting \eqref{eq:likelihood_factorization} into \eqref{eq:posterior_bayes_full},
we obtain
\begin{equation}
    \label{eq:posterior_expanded}p(t\mid \vz) = \frac{ \pi(t)\, p(\vi\mid t)\, p(\xquery\mid
    t)\, \prod_{n=1}^{N}p(x_{n},y_{n}\mid t) }{ \sum_{s\in\gT}\pi(s)\, p(\vi\mid
    s)\, p(\xquery\mid s)\, \prod_{n=1}^{N}p(x_{n},y_{n}\mid s) }.
\end{equation}

To make the dependence on the number of demonstrations explicit, we compare each
task $t$ against the true task $t^{*}$ by forming posterior odds. For any
$t\neq t^{*}$, dividing $p(t\mid \vz)$ by $p(t^{*}\mid \vz)$ gives
\begin{equation}
    \label{eq:posterior_odds_raw}\frac{p(t\mid \vz)}{p(t^{*}\mid \vz)}= \frac{\pi(t)}{\pi(t^{*})}
    \cdot \frac{p(\vi\mid t)}{p(\vi\mid t^{*})}\cdot \frac{p(\xquery\mid t)}{p(\xquery\mid
    t^{*})}\cdot \frac{\prod_{n=1}^{N}p(x_{n},y_{n}\mid t)}{\prod_{n=1}^{N}p(x_{n},y_{n}\mid
    t^{*})}.
\end{equation}

For notational convenience, we separate the contribution of the demonstrations from
all remaining factors. Specifically, we define the \emph{task likelihood ratio
induced by demonstrations}
\begin{equation}
    \label{eq:def_Bt}B_{t}~:=~ \frac{\prod_{n=1}^{N}p(x_{n},y_{n}\mid t)}{\prod_{n=1}^{N}p(x_{n},y_{n}\mid
    t^{*})},
\end{equation}
and the \emph{non-demonstration ratio} (prior, instruction, and query input)
\begin{equation}
    \label{eq:def_At}A_{t}(\vi) ~:=~ \frac{\pi(t)}{\pi(t^{*})}\cdot \frac{p(\vi\mid
    t)}{p(\vi\mid t^{*})}\cdot \frac{p(\xquery\mid t)}{p(\xquery\mid t^{*})}.
\end{equation}
With these definitions, the posterior odds admit the compact form
\begin{equation}
    \label{eq:posterior_odds_representation}\frac{p(t\mid \vz)}{p(t^{*}\mid \vz)}
    ~=~ A_{t}(\vi)\, B_{t}, \qquad t\neq t^{*}.
\end{equation}

Using the normalization condition $\sum_{t\in\gT}p(t\mid \vz)=1$, the posterior
distribution can equivalently be written as
\begin{equation}
    \label{eq:posterior_representation}p(t^{*}\mid \vz) ~=~ \frac{1}{1+\sum_{s\neq
    t^*}A_{s}(\vi)B_{s}}, \qquad p(t\mid \vz) ~=~ \frac{A_{t}(\vi)B_{t}}{1+\sum_{s\neq
    t^*}A_{s}(\vi)B_{s}}, \quad t\neq t^{*}.
\end{equation}
\noindent
Then we present two lemmas.
\begin{lemma}[Exponential Decay of Demonstration Likelihood Ratio]
    \label{lmm:Bt_exp_decay} Defining
    \[
        \alpha := \tfrac12\min_{t\neq t^*}\kappa_{t}>0,
    \]
    there exists $N_{0}<\infty$ such that almost surely, for all $N\ge N_{0}$
    and all $t\neq t^{*}$,
    \[
        B_{t}\le e^{-\alpha N}.
    \]
\end{lemma}
\begin{proof}
    Fix any $t\neq t^{*}$. By the strong law of large numbers and
    \cref{assp:identifiability},
    \begin{equation}
        \label{eq:SLLN_KL}\frac{1}{N}\sum_{n=1}^{N}\log \frac{p(x_{n},y_{n}\mid t)}{p(x_{n},y_{n}\mid
        t^{*})}~\xrightarrow{\mathrm{a.s.}}~ -\kl\!\Big(P_{\rx,\ry\mid t^*}\,\Vert
        \,P_{\rx,\ry\mid t}\Big) ~=:-\kappa_{t},
    \end{equation}
    where $\kappa_{t}>0$. Hence, almost surely, there exists $N_{t}<\infty$ such
    that for all $N\ge N_{t}$,
    \begin{equation}
        \label{eq:Bt_exp_decay_single}B_{t}\le \exp\!\big(-( \kappa_{t}/2)\, N\big
        ).
    \end{equation}
    Since $\gT$ is finite, letting
    \begin{equation}
        \label{eq:alpha_def}\alpha := \frac{1}{2}\min_{t\neq t^*}\kappa_{t}\;>\;0
        ,
    \end{equation}
    we have (after enlarging $N_{0}:=\max_{t\neq t^*}N_{t}$) that almost surely,
    for all $N\ge N_{0}$ and all $t\neq t^{*}$,
    \begin{equation}
        \label{eq:Bt_exp_decay_uniform}B_{t}\le e^{-\alpha N}.
    \end{equation}
\end{proof}

\begin{lemma}[Uniform Boundedness of Non-Demonstration Ratio]
    \label{lmm:At_uniform} There exists a finite constant $A_{\max}$ such that
    \[
        \sup_{\vi\in\dicl}\max_{t\neq t^*}A_{t}(\vi) \le A_{\max}< \infty.
    \]
\end{lemma}

\begin{proof}
    Recall that
    \[
        A_{t}(\vi) =\frac{\pi(t)}{\pi(t^{*})}\cdot \frac{p(\vi\mid t)}{p(\vi\mid
        t^{*})}\cdot \frac{p(\xquery\mid t)}{p(\xquery\mid t^{*})}.
    \]
    Since $\gT$ is finite and $\xquery$ is fixed, the quantities
    $\max_{t\neq t^*}\frac{\pi(t)}{\pi(t^{*})}$ and
    $\max_{t\neq t^*}\frac{p(\xquery\mid t)}{p(\xquery\mid t^{*})}$ are finite
    constants. Thus it suffices to show that
    \[
        \sup_{\vi\in\dicl}\max_{t\neq t^*}\frac{p(\vi\mid t)}{p(\vi\mid t^{*})}<\infty
        .
    \]

    By Assumption~\ref{assp:bounded-instruction}, there exists $G<\infty$ such that
    for all $t\in\gT$ and all $\vi\in\dicl$,
    \[
        \|\nabla_{\vi}\log p(\vi\mid t)\|_{2}\le G.
    \]
    Fix an arbitrary reference instruction $\vi_{0}\in\dicl$. For any $\vi\in\dicl$
    and any $t\in\gT$, by the mean value theorem applied to the differentiable
    function $\log p(\cdot\mid t)$ along the line segment between $\vi_{0}$ and
    $\vi$,
    \[
        \big|\log p(\vi\mid t)-\log p(\vi_{0}\mid t)\big| \le \sup_{\tilde\vi \in
        [\vi_0,\vi]}\|\nabla_{\vi}\log p(\tilde\vi\mid t)\|_{2}\cdot \|\vi-\vi_{0}
        \|_{2}\le G\|\vi-\vi_{0}\|_{2}.
    \]
    Since $\dicl$ is bounded (compact), its diameter is finite:
    \[
        D:=\sup_{\vi\in\dicl}\|\vi-\vi_{0}\|_{2}<\infty.
    \]
    Hence for all $\vi\in\dicl$ and all $t\in\gT$,
    \[
        \log p(\vi\mid t)\le \log p(\vi_{0}\mid t)+GD, \qquad \log p(\vi\mid t)\ge
        \log p(\vi_{0}\mid t)-GD.
    \]
    Applying the same bounds to $t^{*}$ and subtracting gives, for all $\vi\in\dicl$
    and all $t\in\gT$,
    \[
        \log\frac{p(\vi\mid t)}{p(\vi\mid t^{*})}= \log p(\vi\mid t)-\log p(\vi\mid
        t^{*}) \le \big(\log p(\vi_{0}\mid t)-\log p(\vi_{0}\mid t^{*})\big)+2GD.
    \]
    Exponentiating and taking maxima over the finite task set yields
    \[
        \sup_{\vi\in\dicl}\max_{t\neq t^*}\frac{p(\vi\mid t)}{p(\vi\mid t^{*})}\le
        \Big(\max_{t\neq t^*}\frac{p(\vi_{0}\mid t)}{p(\vi_{0}\mid t^{*})}\Big)\,
        e^{2GD}<\infty,
    \]
    where finiteness holds because $\gT$ is finite and $p(\vi_{0}\mid t)>0$ on
    the support of the instruction domain.

    Combining with the finite constants
    $\max_{t\neq t^*}\frac{\pi(t)}{\pi(t^{*})}$ and
    $\max_{t\neq t^*}\frac{p(\xquery\mid t)}{p(\xquery\mid t^{*})}$ gives the
    desired bound with
    \[
        A_{\max}:= \Big(\max_{t\neq t^*}\frac{\pi(t)}{\pi(t^{*})}\Big) \Big(\max_{t\neq
        t^*}\frac{p(\xquery\mid t)}{p(\xquery\mid t^{*})}\Big) \Big(\max_{t\neq t^*}
        \frac{p(\vi_{0}\mid t)}{p(\vi_{0}\mid t^{*})}\Big)e^{2GD}.
    \]
\end{proof}

\subsection{Proof of \cref{thm:decay_1_5}}
We first recall the statement of the theorem for convenience.
\begin{theorem}[Exponential Convergence of Posterior Predictive (Formats 1--5)]
    \label{thm:decay_1_5_app} Consider prompt formats 1--5. Under
    \cref{assp:identifiability,assp:bounded-instruction}, there exist constants $\alpha
    ,\beta>0$ such that, with high probability, for all sufficiently large numbers
    of demonstrations $N$,
    \begin{align}
        \bigl| p(\yquery \mid \vz) - p(\yquery \mid \xquery, t^{*}) \bigr| \le \beta \exp(-\alpha N).
    \end{align}
    Equivalently, for any two instructions $\vi,\vi'$ (with all non-instruction
    parts of $\vz$ fixed), we have
    \begin{align}
        \bigl| p(\yquery \mid \vz(\vi)) - p(\yquery \mid \vz(\vi')) \bigr| \le 2\beta \exp(-\alpha N).
    \end{align}
\end{theorem}
\noindent
We now proceed with the proof.
\begin{proof}
    By definition of the posterior predictive,
    \begin{equation}
        \label{eq:pp_def}p(\yquery\mid \vz) ~=~ \sum_{t\in\gT}p(\yquery\mid \xquery
        ,t)\, p(t\mid \vz).
    \end{equation}
    Subtracting $p(\yquery\mid \xquery,t^{*})$ and using
    $\sum_{t}p(t\mid \vz)=1$ yields
    \begin{align}
        \label{eq:pp_diff_expand}p(\yquery\mid \vz) - p(\yquery\mid \xquery,t^{*}) & = \sum_{t\neq t^*}\big(p(\yquery\mid \xquery,t)-p(\yquery\mid \xquery,t^{*})\big)\, p(t\mid \vz).
    \end{align}
    Since $0\le p(\yquery\mid \xquery,t)\le 1$, we have $|p(\yquery\mid \xquery,t
    )-p(\yquery\mid \xquery,t^{*})|\le 1$ for all $t$, hence
    \begin{equation}
        \label{eq:pp_diff_bound_by_posterior_mass}\big|p(\yquery\mid \vz) - p(\yquery
        \mid \xquery,t^{*})\big| ~\le~ \sum_{t\neq t^*}p(t\mid \vz).
    \end{equation}
    It remains to upper bound the posterior mass on wrong tasks. From \cref{eq:posterior_representation},
    for any $t\neq t^{*}$,
    \begin{equation}
        p(t\mid \vz) ~=~ \frac{A_{t}(\vi)B_{t}}{1+\sum_{s\neq t^*}A_{s}(\vi)B_{s}}
        ~\le~ A_{t}(\vi)B_{t}.
    \end{equation}
    Summing over $t\neq t^{*}$ and applying Lemma~\ref{lmm:Bt_exp_decay} and ~\ref{lmm:At_uniform},
    we obtain that almost surely for all $N\ge N_{0}$,
    \begin{equation}
        \label{eq:wrong_mass_exp_decay}\sum_{t\neq t^*}p(t\mid \vz) ~\le~ \sum_{t\neq
        t^*}A_{t}(\vi)B_{t}~\le~ (|\gT|-1)\,A_{\max}\, e^{-\alpha N}.
    \end{equation}
    Combining \cref{eq:pp_diff_bound_by_posterior_mass,eq:wrong_mass_exp_decay}
    proves the first statement with $\beta := (|\gT|-1)A_{\max}$:
    \[
        \big|p(\yquery\mid \vz)-p(\yquery\mid \xquery,t^{*})\big| \le \beta e^{-\alpha
        N}.
    \]

    For the equivalent statement comparing two instructions $\vi,\vi'$ (with all
    other parts of $\vz$ fixed), we use the triangle inequality:
    \begin{align}
         & \big|p(\yquery\mid \vz(\vi)) - p(\yquery\mid \vz(\vi'))\big|                                                                             \\
         & \le \big|p(\yquery\mid \vz(\vi)) - p(\yquery\mid \xquery,t^{*})\big| + \big|p(\yquery\mid \vz(\vi')) - p(\yquery\mid \xquery,t^{*})\big| \\
         & \le 2\beta e^{-\alpha N},
    \end{align}
    which completes the proof.
\end{proof}

\subsection{Proof of Corollary~\ref{cor:decay_1_5}}
We first recall the statement of the corollary for convenience.
\begin{corollary}[Gradient Decay of Prompt Sensitivity]
    \label{cor:decay_1_5_app} Under the conditions of \cref{thm:decay_1_5}, then
    there exist constants $\alpha',\beta'>0$ such that, with high probability,
    for all sufficiently large $N$,
    \begin{align}
        \bigl\|\nabla_{\vi}p(\yquery \mid \vz)\bigr\| \le \beta' \exp(-\alpha' N).
    \end{align}
\end{corollary}
\noindent
We now proceed with the proof.
\begin{proof}
    We start from \cref{eq:pp_def}. Since $p(\yquery\mid \xquery,t)$ is
    independent of $\vi$ under formats 1--5, differentiating gives
    \begin{equation}
        \label{eq:grad_pp_expand}\nabla_{\vi}p(\yquery\mid \vz) ~=~ \sum_{t\in\gT}
        p(\yquery\mid \xquery,t)\, \nabla_{\vi}p(t\mid \vz),
    \end{equation}
    hence by $0\le p(\yquery\mid \xquery,t)\le 1$ and the triangle inequality,
    \begin{equation}
        \label{eq:grad_pp_triangle}\big\|\nabla_{\vi}p(\yquery\mid \vz)\big\| ~\le
        ~ \sum_{t\in\gT}\big\|\nabla_{\vi}p(t\mid \vz)\big\|.
    \end{equation}

    Next we control $\|\nabla_{\vi}p(t\mid \vz)\|$. Differentiating the softmax-form
    posterior yields the standard identity
    \begin{equation}
        \label{eq:posterior_grad_identity}\nabla_{\vi}p(t\mid \vz) = p(t\mid \vz)
        \!\left( \nabla_{\vi}\log p(\vi\mid t) - \sum_{s\in\gT}p(s\mid \vz)\,\nabla
        _{\vi}\log p(\vi\mid s) \right).
    \end{equation}
    By \cref{assp:bounded-instruction}, $\|\nabla_{\vi}\log p(\vi\mid t)\|\le G$
    for all $t$, thus
    \begin{equation}
        \label{eq:posterior_grad_bound_pt}\|\nabla_{\vi}p(t\mid \vz)\| \le p(t\mid
        \vz)\Big(G + \sum_{s\in\gT}p(s\mid \vz)\,G\Big) \le 2G\,p(t\mid \vz).
    \end{equation}

    Now we sum over $t$. Use $\sum_{t\in\gT}p(t\mid \vz)=1$ to relate the gradient
    of the true task to the others:
    \[
        \nabla_{\vi}p(t^{*}\mid \vz) = -\sum_{t\neq t^*}\nabla_{\vi}p(t\mid \vz),
        \qquad \Rightarrow\quad \|\nabla_{\vi}p(t^{*}\mid \vz)\| \le \sum_{t\neq
        t^*}\|\nabla_{\vi}p(t\mid \vz)\|.
    \]
    Therefore,
    \begin{equation}
        \label{eq:sum_grad_reduce_to_wrong}\sum_{t\in\gT}\|\nabla_{\vi}p(t\mid \vz
        )\| \le 2\sum_{t\neq t^*}\|\nabla_{\vi}p(t\mid \vz)\| \overset{\eqref{eq:posterior_grad_bound_pt}}
        {\le}4G\sum_{t\neq t^*}p(t\mid \vz).
    \end{equation}
    Finally, by \cref{eq:wrong_mass_exp_decay} (proved in the theorem proof), almost
    surely for all $N\ge N_{0}$,
    \[
        \sum_{t\neq t^*}p(t\mid \vz) \le (|\gT|-1)A_{\max}e^{-\alpha N}.
    \]
    Combining with \cref{eq:grad_pp_triangle,eq:sum_grad_reduce_to_wrong} gives
    \[
        \big\|\nabla_{\vi}p(\yquery\mid \vz)\big\| \le 4G(|\gT|-1)A_{\max}\, e^{-\alpha
        N}.
    \]
    Thus the corollary holds with $\alpha'=\alpha$ and
    $\beta' = 4G(|\gT|-1)A_{\max}$.
\end{proof}

\subsection{Proof of Corollary~\ref{cor:decay_1_5_expect}}
We first introduce the notation used in this proof and make an essential realizability
assumption.

\begin{assumption}
    [Log-space realizability and positivity] \label{ass:log_approx} There exists
    $\delta>0$ such that
    \begin{equation}
        \label{eq:delta_lower}q^{*}(\vz)\ge \delta \qquad \text{for all }\vz \text{
        in the support of }\mu_{N},
    \end{equation}
    and for any $\epsilon>0$, there exists a parameter $\theta_{0}$ in the
    considered model class such that
    \begin{equation}
        \label{eq:log_uniform_approx}\sup_{\vz \text{ in the support of }\mu_N}\big
        |\log q_{\theta_0}(\vz)-\log q^{*}(\vz)\big| \le \epsilon .
    \end{equation}
    Moreover, the trained predictor is bounded away from zero on the same support:
    for the parameter $\hat\theta$ returned by training (defined below),
    \begin{equation}
        \label{eq:delta_lower_model}q_{\hat\theta}(\vz)\ge \delta \qquad \text{for
        all }\vz \text{ in the support of }\mu_{N}.
    \end{equation}
\end{assumption}

Assumption~\ref{ass:log_approx} is aligned with the cross-entropy training
objective, which minimizes the (log-space) risk $\E_{\vz\sim\mu_N}\big[-\log q_{\theta}
(\vz)\big]$. Hence, realizability in log-space is the natural notion of approximation
for predictors learned by cross-entropy training. The lower bound \eqref{eq:delta_lower}
avoids singularities of the log-loss, and \eqref{eq:delta_lower_model} ensures that
log-space control can be converted into control in probability space.

\noindent
We now proceed with the proof.
\begin{proof}
    Fix a demonstration count $N$ and let $\mu_{N}$ be the distribution over prompts
    of formats~1--5 with exactly $N$ demonstrations (with the non-instruction parts
    fixed). Write $q^{*}(\vz):=p(\yquery\mid \vz)$ for the Bayesian predictor and
    $q_{\theta}(\vz):=p_{\theta}(\yquery\mid \vz)$ for the transformer predictor.

    Let $\hat\theta$ be the parameter obtained by (approximately) minimizing the
    cross-entropy objective over the same prompt distribution $\mu_{N}$:
    \[
        \hat\theta\in \arg\min_{\theta}\ \E_{\vz\sim\mu_N}\big[-\log q_{\theta}(\vz
        )\big].
    \]
    By definition of risk minimization, the trained solution is not worse than the
    candidate $\theta_{0}$: there exists $\eta_{\mathrm{train}}\ge 0$ capturing
    optimization and statistical error such that
    \begin{equation}
        \label{eq:erm_no_worse}\E_{\vz\sim\mu_N}\big[-\log q_{\hat\theta}(\vz)\big
        ] \le \E_{\vz\sim\mu_N}\big[-\log q_{\theta_0}(\vz)\big] +\eta_{\mathrm{train}}
        .
    \end{equation}

    Define the \emph{log-space approximation error in expectation}
    \begin{equation}
        \label{eq:def_epsN_log}\varepsilon_{N}^{\log}:= \E_{\vz\sim\mu_N}\Big[\big
        |\log q_{\hat\theta}(\vz)-\log q^{*}(\vz)\big|\Big].
    \end{equation}
    We treat $\varepsilon_{N}^{\log}$ as the pretraining/training error term
    induced by cross-entropy training (on $\mu_{N}$), and note that it is expected
    to decrease as the model class becomes more expressive and optimization/statistical
    errors shrink. (The subsequent bound is stated in terms of $\varepsilon_{N}^{\log}$
    and does not require identifying its exact dependence on architecture size.)

    \medskip
    \noindent
    \emph{From log-space error to probability-space error.} For any $a,b\in[\delta
    ,1]$, the mean value theorem applied to $\exp(\cdot)$ gives
    \[
        |a-b| = \big|e^{\log a}-e^{\log b}\big| = e^{\xi}\,|\log a-\log b| \le |\log
        a-\log b|,
    \]
    for some $\xi$ between $\log a$ and $\log b$, where we used $e^{\xi}\le 1$ since
    $\xi\le 0$. Applying this inequality with $a=q_{\hat\theta}(\vz)$ and
    $b=q^{*}(\vz)$ (which lie in $[\delta,1]$ by \eqref{eq:delta_lower} and
    \eqref{eq:delta_lower_model}) and taking expectations yields
    \begin{equation}
        \label{eq:abs_le_log}\E_{\vz\sim\mu_N}\Big[\big|q_{\hat\theta}(\vz)-q^{*}
        (\vz)\big|\Big] \le \varepsilon_{N}^{\log}.
    \end{equation}
    Letting
    \[
        \varepsilon_{N}:= \E_{\vz\sim\mu_N}\Big[\big|q_{\hat\theta}(\vz)-q^{*}(\vz
        )\big|\Big],
    \]
    we obtain $\varepsilon_{N}\le \varepsilon_{N}^{\log}$.

    \medskip
    \noindent
    \emph{Conclude by triangle inequality.} Using the triangle inequality exactly
    as in the proof strategy for the instruction-decay bound,
    \[
        \begin{aligned}
            \E_{\vz\sim\mu_N}\Big[\big|q_{\hat\theta}(\vz(\vi)) - q_{\hat\theta}(\vz(\vi'))\big|\Big] & \le \E_{\vz\sim\mu_N}\Big[\big|q_{\hat\theta}(\vz(\vi)) - q^{*}(\vz(\vi))\big|\Big] + \E_{\vz\sim\mu_N}\Big[\big|q^{*}(\vz(\vi)) - q^{*}(\vz(\vi'))\big|\Big] \\
                                                                                                      & \quad+ \E_{\vz\sim\mu_N}\Big[\big|q^{*}(\vz(\vi')) - q_{\hat\theta}(\vz(\vi'))\big|\Big].
        \end{aligned}
    \]
    By Theorem~\ref{thm:decay_1_5} (applied to the Bayes predictor), the middle term
    is bounded by $2\beta e^{-\alpha N}$. The first and third terms are each
    bounded by
    $\E_{\vz\sim\mu_N}\big[\big|q_{\hat\theta}(\vz)-q^{*}(\vz)\big|\big]=\varepsilon
    _{N}$, since only the instruction component varies and the expectation is over
    the same $\mu_{N}$ with non-instruction parts fixed. Therefore,
    \[
        \E_{\vz\sim\mu_N}\Big[\big|q_{\hat\theta}(\vz(\vi)) - q_{\hat\theta}(\vz(
        \vi'))\big|\Big] \le 2\beta e^{-\alpha N}+ 2\varepsilon_{N},
    \]
    which is the claimed bound \eqref{eq:decay_1_5_expect}.
\end{proof}

%% file: appendix_instruction_generalization.tex
\label{app:instruction-generalization}

In this appendix, we provide a complete proof of the formal Lipschitz generalization
bound stated in \cref{sec:general}. The proof follows the same overall strategy as
the proof sketch in the main text: we first restrict the loss to a one-dimensional
path connecting the test prompt to the pretraining domain, then approximate the
resulting path-restricted loss by a Bernstein polynomial, apply the Remez inequality
to transfer control from the pretraining portion of the path to the full path, and
finally optimize the polynomial degree to obtain the stated bound.
\noindent
We restate \cref{thm:decay_6} in a fully explicit form. Throughout this appendix,
we work under prompt format~6, in which each demonstration is associated with its
own instruction, and we consider the instruction \emph{sequence} as the variable.

The instruction is a length-$N$ sequence whose embedding is denoted by $I\in\mathbb{R}
^{k}$ with $k=Nd$.
\[
    I^{\mathrm{inner}}:= (i^{*},\dots,i^{*})\in\mathbb{R}^{k}, \qquad \di := B_{r}
    (I^{\mathrm{inner}}), \qquad \diout := B_{R}(I^{\mathrm{inner}}).
\]

\begin{theorem}[Failure of convergence of posterior predictive (Format 6), formal
version of \cref{thm:decay_6}]
    \label{thm:decay_6_formal} Consider prompt format~6, under
    \cref{assp:identifiability,assp:bounded-instruction}, and fix radii $0<r<R$.
    Almost surely, there exist constants $\alpha>0$ and $N_{0}<\infty$ such that
    for all $N\ge N_{0}$ and all polynomial degrees $n\ge 1$,
    \begin{equation}
        \label{eq:decay6_formal_main}\|f-f^{*}\|_{L^\infty(\diout)}\le \exp(C_{1}
        kn)\Big(\beta(N)\,e^{-\alpha N}+\varepsilon_{n}\Big) +\varepsilon_{n},
    \end{equation}
    where $C_{1}=\log(2R/r)$, $\beta(N):=(|\gT|-1)A_{\max}(N)$ with $A_{\max}(N)=
    O(e^{2G\sqrt{N}r})$ depending on $N$, and
    \begin{equation}
        \label{eq:decay6_eps}\varepsilon_{n}\le c_{0}GR\sqrt{d}\frac{N}{\sqrt{n}}
        ,
    \end{equation}
    with an absolute constant $c_{0}>0$.
\end{theorem}
\subsection{Preliminaries}
We follow the Bayesian decomposition used in \cref{thm:decay_1_5}, and then show
how the high-dimensional instruction variable in format~6 prevents extending
inner-region control to the outer region.

\paragraph{Posterior predictive decomposition.}
By definition,
\[
    f(I) = p(\yquery \mid \vz(I)) = \sum_{t\in\gT}p(\yquery\mid \xquery,t)\,p (t\mid
    \vz(I)).
\]
Since $p(\yquery\mid \xquery,t)$ does not depend on $I$ under format~6, subtracting
$f^{*}=p(\yquery\mid \xquery,t^{*})$ and using $\sum_{t}p(t\mid \vz)=1$ gives
\begin{align}
    \label{eq:pp_gap_by_wrong_mass}|f(I)-f^{*}| & = \Big|\sum_{t\neq t^*}\big(p(\yquery\mid\xquery,t)-p(\yquery\mid\xquery,t^{*})\big) \,p(t\mid \vz(I))\Big| \le \sum_{t\neq t^*}p(t\mid \vz(I)).
\end{align}
Thus it suffices to upper bound the posterior mass on wrong tasks.

\paragraph{Posterior odds and task concentration (inner region).}
Bayes' rule yields, for each $t\in\gT$,
\[
    p(t\mid \vz(I)) \propto \pi(t)\,p(I\mid t)\,p(\xquery\mid t) \prod_{n=1}^{N}p
    (x_{n},y_{n}\mid t).
\]
Fix $t\neq t^{*}$. Forming posterior odds against $t^{*}$ as in
\cref{thm:decay_1_5} gives
\begin{equation}
    \label{eq:odds_format6}\frac{p(t\mid \vz(I))}{p(t^{*}\mid \vz(I))}= \underbrace{%
    \frac{\pi(t)}{\pi(t^{*})}\cdot \frac{p(I\mid t)}{p(I\mid t^{*})}\cdot
    \frac{p(\xquery\mid t)}{p(\xquery\mid t^{*})} }_{=:\,A_t(I)}\cdot \underbrace{%
    \frac{\prod_{n=1}^{N}p(x_{n},y_{n}\mid t)}{\prod_{n=1}^{N}p(x_{n},y_{n}\mid
    t^{*})}
    }_{=:\,B_t}.
\end{equation}

By \cref{assp:identifiability} and the strong law of large numbers (exactly the
same argument as in \cref{thm:decay_1_5}), almost surely there exist $\alpha>0$
and $N_{0}<\infty$ such that for all $N\ge N_{0}$ and all $t\neq t^{*}$,
\begin{equation}
    \label{eq:Bt_decay}B_{t}\le e^{-\alpha N}.
\end{equation}

It remains to control $A_{t}(I)$ on the inner ball
$\di=B_{r}(I^{\mathrm{inner}})$. Since each demonstration carries its own instruction,
the joint instruction likelihood factorizes as
$p(I\mid t)=\prod_{n=1}^{N}p(i_{n}\mid t)$, so $\log p(I\mid t)=\sum_{n=1}^{N}\log
p(i_{n}\mid t)$. By \cref{assp:bounded-instruction}, $\|\nabla_{i_n}\log p(i_{n}\mid
t)\|_{2}\le G$ for every $n$ and $t$, hence
\[
    \|\nabla_{I}\log p(I\mid t)\|_{2}= \sqrt{\sum_{n=1}^{N}\|\nabla_{i_n}\log p(i_{n}\mid
    t)\|_{2}^{2}}\le G\sqrt{N}.
\]
Thus $\log p(I\mid t)$ is $G\sqrt{N}$-Lipschitz on $\mathbb{R}^{k}$. For any $I\in
\di$, applying the mean value theorem to both $t$ and $t^{*}$ and subtracting
gives
\[
    \Big|\log\frac{p(I\mid t)}{p(I\mid t^{*})}-\log\frac{p(I^{\mathrm{inner}}\mid
    t)}{p(I^{\mathrm{inner}}\mid t^{*})}\Big| \le 2G\sqrt{N}\,\|I-I^{\mathrm{inner}}
    \|_{2}\le 2G\sqrt{N}\,r.
\]
Exponentiating,
\begin{equation}
    \label{eq:At_inner_bound}\sup_{I\in\di}\frac{p(I\mid t)}{p(I\mid t^{*})}\le e
    ^{2G\sqrt{N}r}\cdot\frac{p(I^{\mathrm{inner}}\mid t)}{p(I^{\mathrm{inner}}\mid
    t^{*})}.
\end{equation}
Therefore $\sup_{I\in\di}A_{t}(I)$ is finite for each $t$ and $N$, and since
$\gT$ is finite we may define
\[
    A_{\max}(N) := \max_{t\neq t^*}\sup_{I\in\di}A_{t}(I) = O\!\left(e^{2G\sqrt{N}r}
    \right).
\]
Combining \eqref{eq:odds_format6}--\eqref{eq:Bt_decay} and using $p(t\mid\vz) \le
A_{t}(I)B_{t}$ (since the posterior denominator is at least~$1$), we obtain that
almost surely, for all $N\ge N_{0}$ and all $I\in\di$,
\begin{equation}
    \label{eq:wrong_mass_inner}\sum_{t\neq t^*}p(t\mid \vz(I)) \le \sum_{t\neq t^*}
    A_{t}(I)B_{t}\le (|\gT|-1)\,A_{\max}(N)\,e^{-\alpha N}.
\end{equation}
Plugging \eqref{eq:wrong_mass_inner} into \eqref{eq:pp_gap_by_wrong_mass} yields
the inner-region bound:
\begin{equation}
    \label{eq:inner_decay_derived}\|f-f^{*}\|_{L^\infty(\di)}\le \beta(N)\,e^{-\alpha
    N}, \qquad \beta(N) := (|\gT|-1)\,A_{\max}(N).
\end{equation}
Note that $\beta(N)=O(e^{2G\sqrt{N}r})$ grows with $N$, so the inner-region
bound itself does not imply convergence as $N\to\infty$.

\paragraph{Multivariate Bernstein approximation on a bounding cube.}
To define a polynomial approximation over $\diout$, we embed the ball
$\diout = B_{R}(I^{\mathrm{inner}})$ into the bounding cube
$Q := I^{\mathrm{inner}}+ [-R,R]^{k}$. We map $Q$ affinely onto the unit cube $[0
,1]^{k}$ via the bijection $\Phi(I)_{j}$
and consider the pullback $h:=(f-f^{*})\circ\Phi^{-1}$ on $[0,1]^{k}$.

For an integer parameter $n\ge 1$, the $k$-variate Bernstein polynomial (tensor product
of 1D Bernstein polynomials) is defined by
\begin{equation}
    \label{eq:def_multivar_bernstein}(B_{n}^{(k)}h)(x) := \sum_{j_1=0}^{n}\cdots
    \sum_{j_k=0}^{n}h\!\Big(\frac{j_{1}}{n},\dots,\frac{j_{k}}{n}\Big) \prod_{\ell=1}
    ^{k}\binom{n}{j_\ell}x_{\ell}^{\,j_\ell}(1-x_{\ell})^{\,n-j_\ell}, \qquad x\in
    [0,1]^{k}.
\end{equation}
Let $p_{n}:=B_{n}^{(k)}h$ (identified with a function on $Q$ via $\Phi$). Note that
the highest monomial in \eqref{eq:def_multivar_bernstein} is
$\prod_{\ell=1}^{k}x_{\ell}^{n}$, so $p_{n}$ has a \emph{total degree} of $nk$.

The coefficients in \eqref{eq:def_multivar_bernstein} form a convex combination,
so the Bernstein polynomial at $x$ equals the expectation of $h$ at the random
grid point $S/n$, where $S=(S_{1},\dots,S_{k})$ with $S_{j}\sim\mathrm{Binomial}(
n,x_{j})$ independently:
\begin{equation}
    \label{eq:bernstein_expectation}(B_{n}^{(k)}h)(x) = \mathbb{E}\!\left[h\! \left
    (\frac{S}{n}\right)\right].
\end{equation}
Assuming $h$ is Lipschitz with constant $L_{h}$ (determined below), Jensen's inequality
yields
\[
    |(B_{n}^{(k)}h)(x)-h(x)| \le L_{h}\,\mathbb{E}\Big\|\frac{S}{n}-x\Big\| \le L
    _{h}\sqrt{\mathbb{E}\Big\|\frac{S}{n}-x\Big\|^{2}}.
\]
Since $\mathrm{Var}(S_{j}/n)=x_{j}(1-x_{j})/n\le 1/(4n)$ and the coordinates are
independent,
\[
    \mathbb{E}\Big\|\frac{S}{n}-x\Big\|^{2}=\sum_{j=1}^{k}\mathrm{Var}\!\Big( \frac{S_{j}}{n}
    \Big) \le\frac{k}{4n}.
\]
Hence, bounding the error over the cube $Q$ (and thereby over the inscribed ball
$\diout$), we have
$\varepsilon_{n}:=\|p_{n}-h\|_{L^\infty([0,1]^k)}\le L_{h}\sqrt{\frac{k}{4n}}$.

\subsection{Proof of \cref{thm:decay_6}}
\begin{proof}
    Under format~6, $I=(i_{1},\dots,i_{N})\in\mathbb{R}^{k}$ is a concatenation
    of $N$ per-demonstration instruction embeddings $i_{n}\in\mathbb{R}^{d}$, with
    $k=Nd$. Using the factorization
    $\log p(I\mid t)=\sum_{n=1}^{N}\log p(i_{n}\mid t)$ established above, differentiating
    with respect to the $n$-th block gives $\nabla_{i_n}\log p(I\mid t)=\nabla_{i_n}
    \log p(i_{n}\mid t)$, and \cref{assp:bounded-instruction} yields
    $\|\nabla_{i_n}\log p(I\mid t)\|_{2}\le G$ for all $n$ and $t$. The softmax-gradient
    identity applied to each block gives
    \[
        \nabla_{i_n}p(t\mid\vz) = p(t\mid\vz)\Bigl( \nabla_{i_n}\log p(I\mid t) -
        \sum_{s\in\gT}p(s\mid\vz)\,\nabla_{i_n}\log p(I\mid s) \Bigr),
    \]
    so $\|\nabla_{i_n}p(t\mid\vz)\|_{2}\le 2G\,p(t\mid\vz)$. Summing over $t$,
    \[
        \|\nabla_{i_n}f(I)\|_{2}\le\sum_{t\in\gT}2G\,p(t\mid\vz(I))=2G \qquad\text{for
        each }n.
    \]
    Since the $N$ coordinate blocks are orthogonal directions in $\mathbb{R}^{k}$,
    \[
        \|\nabla_{I}f(I)\|_{2}=\sqrt{\sum_{n=1}^{N}\|\nabla_{i_n}f(I)\|_{2}^{2}}\le
        2G\sqrt{N}.
    \]
    Hence $f-f^{*}$ is $2G\sqrt{N}$-Lipschitz on $\mathbb{R}^{k}$. The inverse map
    $\Phi^{-1}(x) = I^{\mathrm{inner}}+ 2R(x - \frac{1}{2})$ has a Lipschitz constant
    of $2R$. Thus, its pullback $h=(f-f^{*})\circ\Phi^{-1}$ is Lipschitz with constant
    $L_{h}= 4RG\sqrt{N}$. Substituting into the Bernstein error bound gives
    \[
        \varepsilon_{n}\le 4RG\sqrt{N}\cdot\sqrt{\frac{Nd}{4n}}= c_{0}GR\sqrt{d}\cdot
        \frac{N}{\sqrt{n}},
    \]
    where $c_{0}=2$, yielding \eqref{eq:decay6_eps}.

    On the inner ball $\di$,
    \[
        \|p_{n}\|_{L^\infty(\di)}\le\|f-f^{*}\|_{L^\infty(\di)}+\varepsilon_{n}\le
        \beta(N)\,e^{-\alpha N}+\varepsilon_{n}.
    \]
    Recall that $p_{n}$ is a polynomial of total degree $nk$. Because $\di$ and
    $\diout$ are concentric balls, the restriction of $p_{n}$ to any line passing
    through $I^{\mathrm{inner}}$ is a 1D polynomial of degree at most $nk$.
    Applying the standard 1D Chebyshev extremum property to the segments
    $[-r, r]$ and $[-R, R ]$ along any such ray gives
    \[
        \|p_{n}\|_{L^\infty(\diout)}\le T_{nk}\!\left(\frac{R}{r}\right) \|p_{n}\|
        _{L^\infty(\di)},
    \]
    where $T_{nk}$ denotes the Chebyshev polynomial of the first kind of degree
    $n k$. Using the bound $T_{m}(x)\le(2x)^{m}$ for $x\ge 1$ yields
    \[
        \|p_{n}\|_{L^\infty(\diout)}\le\left(\frac{2R}{r}\right)^{nk}\|p_{n}\|_{L^\infty(\di)}
        =\exp(C_{1}kn)\,\|p_{n}\|_{L^\infty(\di)}, \qquad C_{1}:=\log(2R/r).
    \]
    Therefore,
    \[
        \|p_{n}\|_{L^\infty(\diout)}\le\exp(C_{1}kn)\Big(\beta(N)\,e^{-\alpha N}+
        \varepsilon_{n}\Big).
    \]
    By the triangle inequality
    $\|f-f^{*}\|_{L^\infty(\diout)}\le\|p_{n}\|_{L^\infty(\diout)}+\varepsilon_{n}$,
    we obtain \eqref{eq:decay6_formal_main}.

    Note that $\varepsilon_{n}=c_{0}GR\sqrt{d}\cdot N/\sqrt{n}$ grows with $N$
    for any fixed $n$, and requires $n=\Omega(N^{2})$ to remain bounded. However,
    choosing $n=\Omega(N^{2})$ causes the amplification factor to satisfy $\exp(C
    _{1}kn)=\exp (\Omega(dN^{3}))$, which diverges super-exponentially. Hence no
    choice of $n$ simultaneously controls both terms, and the bound \eqref{eq:decay6_formal_main}
    cannot be made to converge to $0$ as $N\to\infty$.
\end{proof}

\subsection{Proof of Corollary~\ref{cor:decay_6}}

To establish stability guarantees on the extended domain $\diout$, we rely on the
following tool.

\begin{lemma}[Multivariate Markov inequality \citep{Wilhelmsen1974}]
    \label{thm:markov} Let $T\subset\mathbb{R}^{k}$ be a convex body with diameter
    $w_{T}$, and let $p:\mathbb{R}^{k}\to\mathbb{R}$ be a polynomial of total
    degree at most $m$. Then
    \begin{equation}
        \|\nabla p\|_{L^{\infty}(T)}\le \frac{4m^{2}}{w_{T}}\|p\|_{L^{\infty}(T)}
        .
    \end{equation}
\end{lemma}

To upper bound $\|\nabla f\|$ on $\diout$ using a polynomial surrogate, uniform $C
^{0}$ approximation is insufficient in general. We therefore impose a mild
$C^{1}$-approximability condition.

\begin{assumption}
    [Uniform $C^{1}$ polynomial approximability on $\diout$] \label{ass:C1_approx}
    There exists a polynomial $p_{n}$ of total degree at most $nk$ such that
    \[
        \|f-p_{n}\|_{L^\infty(\diout)}\le\varepsilon_{n}, \qquad \|\nabla f-\nabla
        p_{n}\|_{L^\infty(\diout)}\le\delta_{n}.
    \]
\end{assumption}

\begin{proof}
    Recall that $\diout=B_{R}(I^{\mathrm{inner}})\subset\mathbb{R}^{k}$ with $k=N
    d$, so its diameter is $w_{\diout}=2R$. Let $p_{n}$ be the polynomial surrogate
    of total degree $m=nk$ from \cref{ass:C1_approx}.

    By the triangle inequality and the inner-region bound \eqref{eq:inner_decay_derived},
    \begin{equation}
        \label{eq:pn_inner_control}\|p_{n}\|_{L^\infty(\di)}\le \|f-f^{*}\|_{L^\infty(\di)}
        + \|f-p_{n}\|_{L^\infty(\di)}\le \beta(N)\,e^{-\alpha N}+ \varepsilon_{n}
        .
    \end{equation}

    Since $\di$ and $\diout$ are concentric balls, the restriction of $p_{n}$ to
    any line through $I^{\mathrm{inner}}$ is a one-dimensional polynomial of degree
    at most $nk$. Applying the standard Chebyshev growth bound on the concentric
    intervals $[-r,r]$ and $[-R,R]$ yields
    \[
        \|p_{n}\|_{L^\infty(\diout)}\le T_{nk}\!\left(\frac{R}{r}\right) \|p_{n}\|
        _{L^\infty(\di)},
    \]
    where $T_{nk}$ is the Chebyshev polynomial of the first kind. Using $T_{m}(x)
    \le(2x)^{m}$ for $x\ge1$ gives
    \begin{equation}
        \label{eq:pn_outer_amplified}\|p_{n}\|_{L^\infty(\diout)}\le \exp(C_{1}kn
        ) \big( \beta(N)\,e^{-\alpha N}+ \varepsilon_{n}\big), \qquad C_{1}=\log(
        2R/r).
    \end{equation}

    Applying \cref{thm:markov} with $m=nk$ and $w_{\diout}=2R$,
    \begin{equation}
        \label{eq:grad_pn_markov_clean}\|\nabla p_{n}\|_{L^\infty(\diout)}\le \frac{4(nk)^{2}}{2R}
        \|p_{n}\|_{L^\infty(\diout)}= \frac{2n^{2}k^{2}}{R}\|p_{n}\|_{L^\infty(\diout)}
        .
    \end{equation}

    Substituting \eqref{eq:pn_outer_amplified} into \eqref{eq:grad_pn_markov_clean}
    yields
    \begin{equation}
        \|\nabla p_{n}\|_{L^\infty(\diout)}\le \frac{2n^{2}k^{2}}{R}\exp(C_{1}kn)
        \Big( \beta(N)\,e^{-\alpha N}+ \varepsilon_{n}\Big).
    \end{equation}

    Using the $C^{1}$ approximation condition and the triangle inequality,
    \[
        \|\nabla f\|_{L^\infty(\diout)}\le \|\nabla p_{n}\|_{L^\infty(\diout)}+ \delta
        _{n},
    \]
    we obtain
    \begin{equation}
        \label{eq:grad_final_clean}\|\nabla f\|_{L^\infty(\diout)}\le \frac{2n^{2}k^{2}}{R}
        \exp(C_{1}kn) \Big( \beta(N)\,e^{-\alpha N}+ \varepsilon_{n}\Big) + \delta
        _{n}.
    \end{equation}

    Taking
    \[
        n = \Theta(G^{2}R^{2}dN^{2})
    \]
    controls the Bernstein approximation error $\varepsilon_{n}=O(1)$ from \cref{thm:decay_6_formal}.
    Substituting $k=Nd$ gives the stated scaling.
\end{proof}

%% file: appendix_experiment.tex
\label{app:experiment}

In this section we provide detailed introduction of our experiments setup and additional
experimental results to supplement those in the main text.

\subsection{Pretraining Dependence}
\label{app:pretrain_error} To isolate the role of \emph{intrinsic ICL capability},
we construct three variants of \texttt{Llama3.2-3B} that differ only in the
amount and \emph{format} of additional pretraining on a six-digit addition corpus.
Throughout, we use the same evaluation protocol in the main text (6--10 digit addition).

\textbf{Base (Weak).} We use the original \texttt{Llama3.2-3B} checkpoint
without any additional training.

\textbf{Fine-tuned (Strong).} We randomly generate $20{,}000$ six-digit addition
instances and format each instance as a supervised instruction-following sample in
JSONL, each training example is serialized as
\[
    \texttt{Instruction: \ldots\ \ \ Input: \ldots\ \ \ Response: \ldots }
\]
and we fine-tune the model using LoRA on this dataset.

\textbf{Fine-tuned (Medium).} We use the \emph{same} $20{,}000$ generated
instances, but concatenate all serialized training examples into a single long text
sequence and fine-tune on this single sequence. This construction preserves the
token-level training signal while substantially reducing the effective diversity
and sampling of training contexts, leading to a weaker improvement in intrinsic ICL
capability compared to the sample-level training above.

Both \textbf{Strong} and \textbf{Medium} variants are trained with the same LoRA
configuration and optimizer settings. We apply LoRA to attention and MLP
projection modules. We train
a causal LM objective with labels equal to the input tokens (standard next-token
prediction over the serialized prompt). Unless otherwise stated, we use a maximum
sequence length of $512$ tokens with padding/truncation.

\subsection{Demonstration Choice}
\label{app:demo}

In this section, out goal is to isolate how the informativeness of demonstrations
affects in-context generalization, under the same prompt template, the same query
distribution, and the same number of demonstrations.

\paragraph{Tasks.}
We consider two retrieval-style tasks constructed as ``$(x, ?, y)$'' completion:
(i) \textbf{Sports}: $x$ is a country and $y$ is its designated sport label in
our benchmark; (ii) \textbf{People}: $x$ is a company and $y$ is a canonical person
name associated with that company in our benchmark (e.g., a founder). The model
is queried with an unseen $x_{\text{test}}$ and asked to predict the
corresponding $y_{\text{test}}$.

\paragraph{Ambiguous vs.\ identifying demonstration pools.}
For each task, we build two disjoint pools of demonstration pairs: \emph{ambiguous}
and \emph{identifying}. Both pools contain $10$ labeled pairs $(x,y)$. Intuitively,
the ambiguous pool contains examples that can be explained by multiple plausible
latent rules (hence providing weaker evidence about which rule the prompt instantiates),
whereas the identifying pool contains examples that more strongly pin down a unique
rule. Importantly, the \emph{test set} (query distribution) is kept the same
across the two conditions; only the demonstration pool differs.

\paragraph{Prompt template.}
Given a demonstration set $\{(x_{j}, y_{j})\}_{j=1}^{k}$, we format the prompt as
\[
    \texttt{Question: (}x_{1}, ?, y_{1}\texttt{), \ldots, (}x_{k}, ?, y_{k}\texttt
    {). }x_{\text{test}}, ?,
\]
and instruct the model to output a single answer token span in English (sport
name for Sports; person name for People).

\paragraph{Evaluation.}
For each task and each condition (ambiguous / identifying), we generate $100$
prompts per random seed by sampling $k\in\{2,4,8\}$ uniformly and sampling $k$
demonstrations uniformly without replacement from the corresponding pool. The query
pair $(x_{\text{test}}, y_{\text{test}})$ is sampled from a fixed test list with
ground-truth labels. Model outputs are parsed into JSON and normalized by
lowercasing and removing punctuation; accuracy is computed by exact match with the
normalized ground truth. We repeat the entire procedure for $5$ random seeds and
report mean $\pm$ standard deviation across seeds.

\paragraph{Results.}
Tables~\ref{tab:demo_choice_sports}--\ref{tab:demo_choice_people} summarize the
accuracy under the two demonstration pools as a function of $k$ on \texttt{Qwen-235B-A22B}
and \texttt{Qwen-30B-A3B}, respectively. Across both tasks, the identifying pool
consistently yields higher accuracy than the ambiguous pool, and the performance
gap generally increases with more demonstrations.

\begin{table}[htbp]
    \centering
    \caption{\textbf{Demonstration Experiment.} Task: sports, repeated 5
    times with different random seeds.}
    \label{tab:demo_choice_sports} \resizebox{\linewidth}{!}{
    \begin{tabular}{ccccc}
        \toprule \textbf{Demo number} & \textbf{Ambiguous Acc (235B)} & \textbf{Identifying Acc (235B)} & \textbf{Ambiguous Acc (30B)} & \textbf{Identifying Acc (30B)} \\
        \midrule 2                    & $0.265 \pm 0.073$             & $0.421 \pm 0.080$               & $0.122 \pm 0.029$            & $0.176 \pm 0.062$              \\
        4                             & $0.267 \pm 0.084$             & $0.542 \pm 0.048$               & $0.151 \pm 0.078$            & $0.410 \pm 0.059$              \\
        8                             & $0.202 \pm 0.037$             & $0.602 \pm 0.043$               & $0.164 \pm 0.039$            & $0.559 \pm 0.074$              \\
        \bottomrule
    \end{tabular}
    }
\end{table}

\begin{table}[htbp]
    \centering
    \caption{\textbf{Demonstration Experiment.} Task: People, repeated 5
    times with different random seeds.}
    \label{tab:demo_choice_people} \resizebox{\linewidth}{!}{
    \begin{tabular}{ccccc}
        \toprule \textbf{Demo number} & \textbf{Ambiguous Acc (235B)} & \textbf{Identifying Acc (235B)} & \textbf{Ambiguous Acc (30B)} & \textbf{Identifying Acc (30B)} \\
        \midrule 2                    & $0.713 \pm 0.074$             & $0.940 \pm 0.049$               & $0.735 \pm 0.083$            & $0.908 \pm 0.059$              \\
        4                             & $0.708 \pm 0.096$             & $0.931 \pm 0.027$               & $0.775 \pm 0.037$            & $0.950 \pm 0.058$              \\
        8                             & $0.648 \pm 0.008$             & $0.845 \pm 0.044$               & $0.820 \pm 0.027$            & $0.990 \pm 0.001$              \\
        \bottomrule
    \end{tabular}
    }
\end{table}

\subsection{Prompt Sensitivity vs. Demonstration Count}
\label{app:experiment_instruction_variation} We study how token-level numerical
predictions respond to changes in the in-context operator pattern while keeping
the underlying query answer fixed. We evaluate \texttt{Qwen3-30B-A3B} model and
\texttt{Qwen3-235B-A22B} model.

\paragraph{Prompt regimes.}
Each evaluation prompt consists of a fixed-format arithmetic query paired with a
fixed target answer (e.g., \texttt{1345}). We construct three regimes that
differ only in the operator pattern appearing in the demonstrations: (i) \emph{addition-only},
where all in-context examples use \texttt{+}; (ii) \emph{multiplication-only}, where
all examples use \texttt{$\times$}; and (iii) \emph{mixed-symbol}, where different
operators are interleaved across demonstrations. Across all regimes, the query problem
and its correct final answer remain identical.

\paragraph{Token-level score extraction.}
For each prompt, we extract token-level scores corresponding to the correct digits
of the target answer. Concretely, for a target digit string $a_{1}a_{2}\cdots a_{L}$,
we record, at each output position $\ell$, the model-assigned log-probability of
the correct digit $a_{\ell}$. These per-digit log-probabilities serve as a fine-grained
proxy for the model’s confidence in the correct numerical prediction at each position.
Specifically, in the figures, Char~$1 ,3,4 ,5$ denotes the score trace
associated with the $1,2,3,4$-th digit position of the fixed target answer
\texttt{1345}.

\paragraph{Findings.}
\cref{fig:add_logit} and \cref{fig:add_logit_a3b} present the token-level score
trajectories for the \texttt{Qwen3-235B-A22B} and \texttt{Qwen3-30B-A3B} models,
respectively, under the three operator regimes. Corresponding score gradients are
shown in \cref{fig:add_logit_235b_gradient} and \cref{fig:add_logit_a3b_gradient}.
In homogeneous-symbol regimes (addition-only or multiplication-only), per-digit score
trajectories are relatively smooth and exhibit small discrete gradients,
indicating stable symbolic behavior. In contrast, the mixed-symbol regime
produces larger variance and more frequent sign changes in the gradients, suggesting
increased instability when the operator distribution shifts within the context.
These patterns are consistent across model scales, although the larger \texttt{Qwen3-235B-A22B}
model exhibits systematically smaller gradient magnitudes, indicating improved robustness
to symbolic heterogeneity.


\begin{figure}[htbp]
    \centering
    \begin{subfigure}
        [t]{0.32\textwidth}
        \centering
        \includegraphics[width=\textwidth]{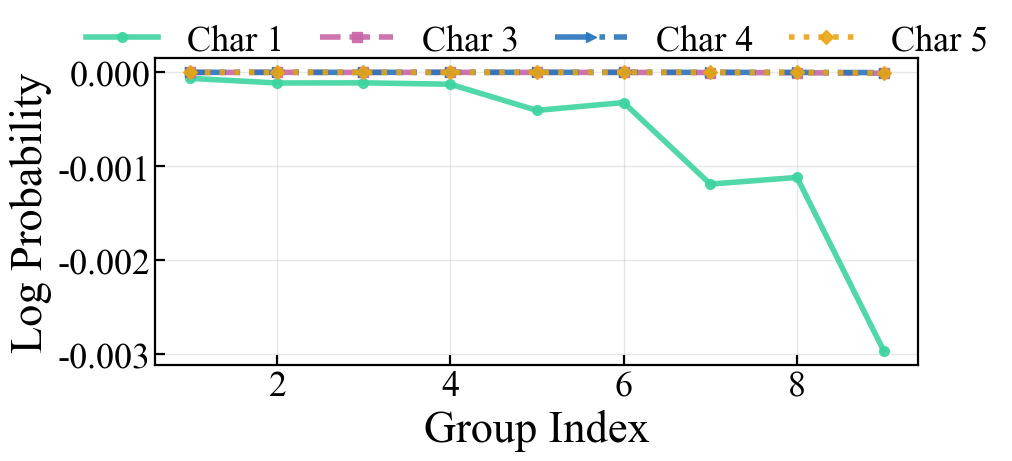}
        \caption{Equivalent Instructions}
        \label{fig:correct_logit_235b}
    \end{subfigure}
    \hfill
    \begin{subfigure}
        [t]{0.32\textwidth}
        \centering
        \includegraphics[width=\textwidth]{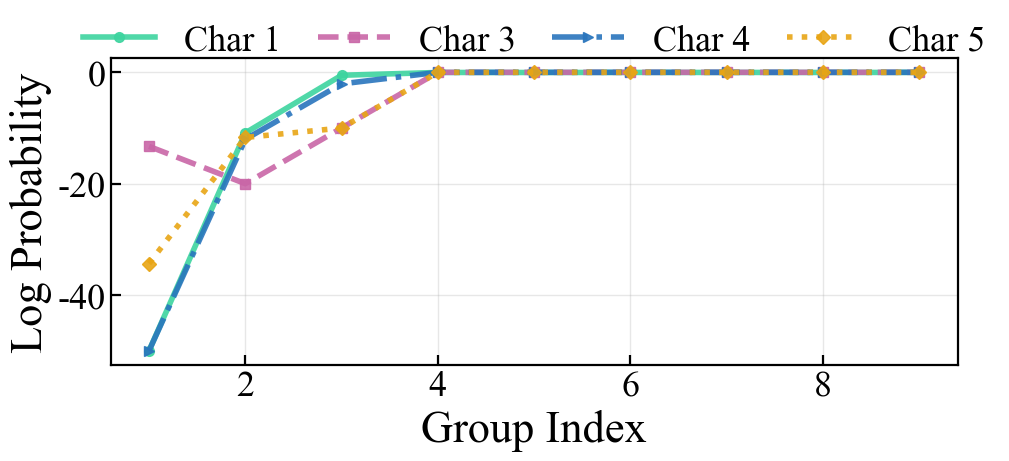}
        \caption{Consistent Instructions}
        \label{fig:consistent_logit_235b}
    \end{subfigure}
    \hfill
    \begin{subfigure}
        [t]{0.32\textwidth}
        \centering
        \includegraphics[width=\textwidth]{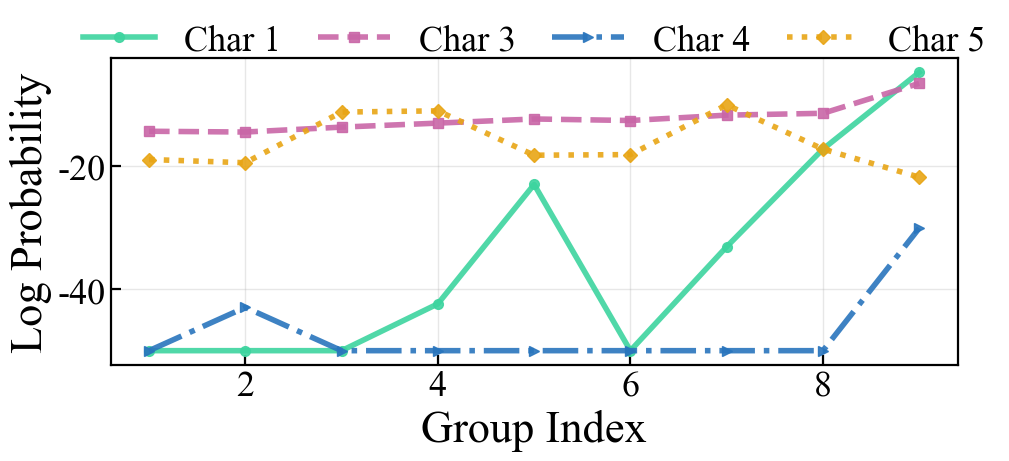}
        \caption{Inconsistent Instructions}
        \label{fig:inconsistent_logit_235b}
    \end{subfigure}
    \caption{\textbf{Posterior confidence under instruction variation.} (\texttt{Qwen-235B-A22B-Instruct})
    (a) Equivalent instructions yield stable, high-confidence predictions. (b)
    Incorrect but consistent instructions allow confidence to recover as more
    demonstrations are added. (c) Incorrect and inconsistent instructions lead to
    persistently low and unstable confidence despite increasing context length.}

    \label{fig:add_logit}
\end{figure}

\begin{figure}[htbp]
    \centering
    \includegraphics[width=0.32\textwidth]{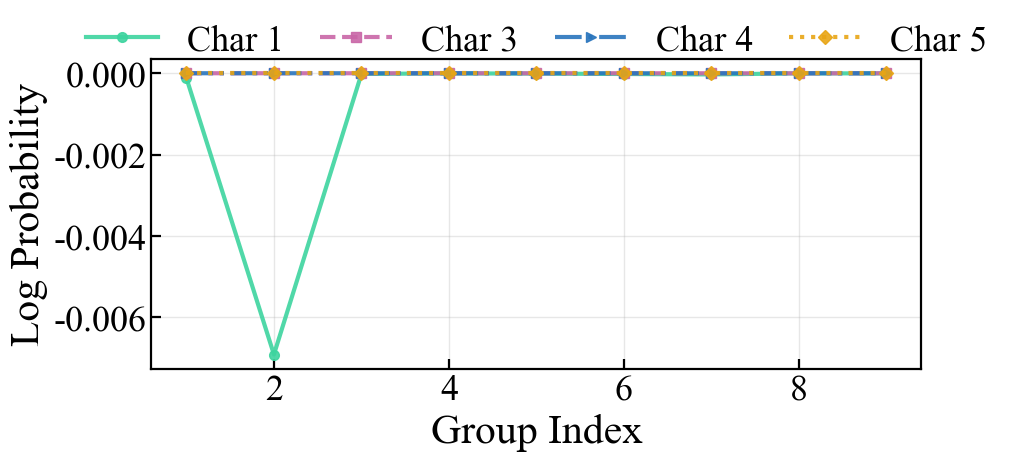}
    \includegraphics[width=0.32\textwidth]{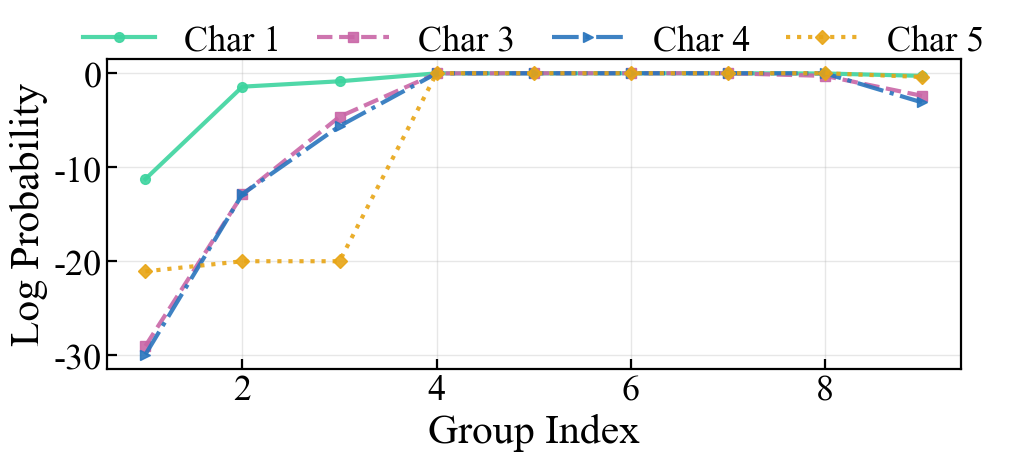}
    \includegraphics[width=0.32\textwidth]{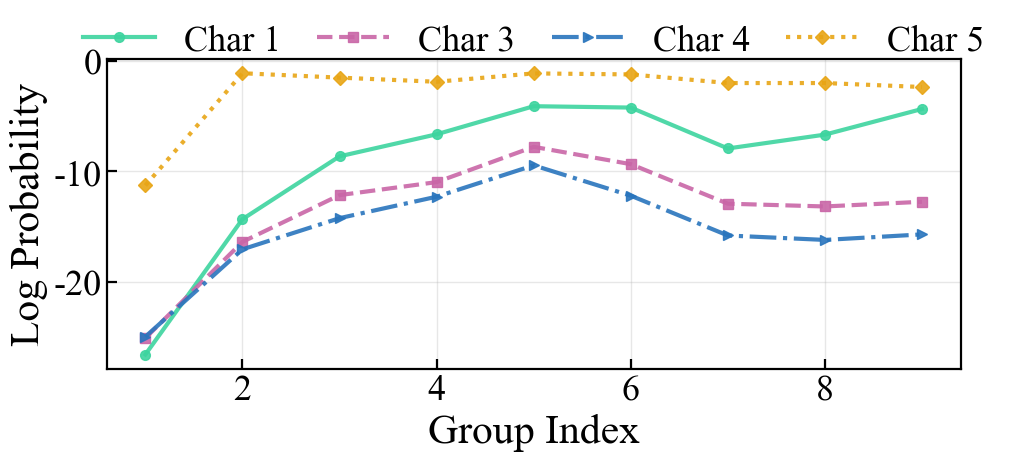}
    \caption{\textbf{Posterior confidence under instruction variation.} (\texttt{Qwen-30B-A3B-Instruct})
    (a) Equivalent instructions yield stable, high-confidence predictions. (b)
    Incorrect but consistent instructions allow confidence to recover as more
    demonstrations are added. (c) Incorrect and inconsistent instructions lead to
    persistently low and unstable confidence despite increasing context length.}
    \label{fig:add_logit_a3b}
\end{figure}

\begin{figure}[htbp]
    \centering
    \includegraphics[width=0.32\textwidth]{
        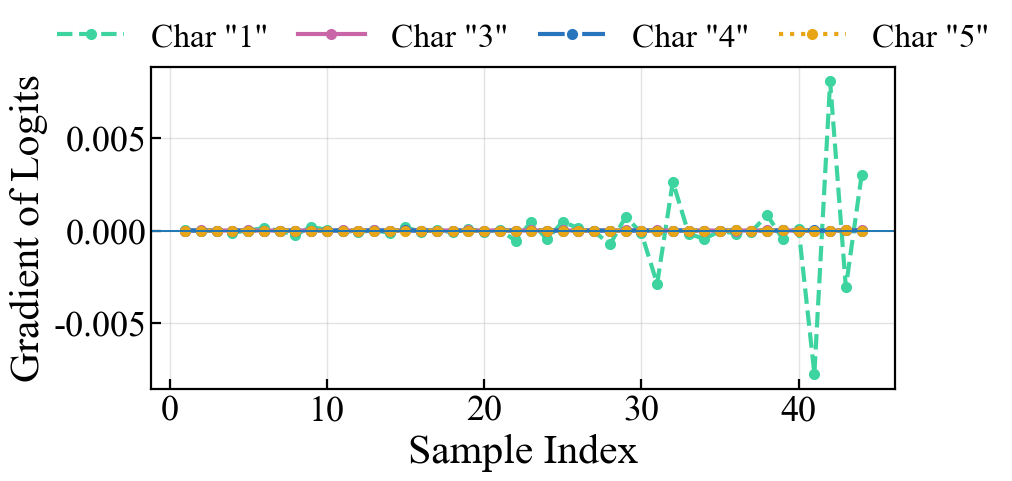
    }
    \includegraphics[width=0.32\textwidth]{
        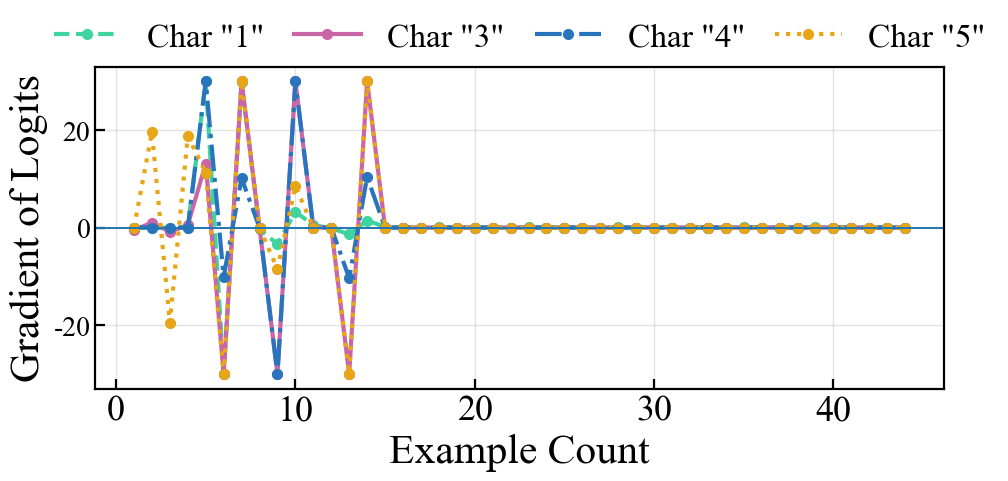
    }
    \includegraphics[width=0.32\textwidth]{
        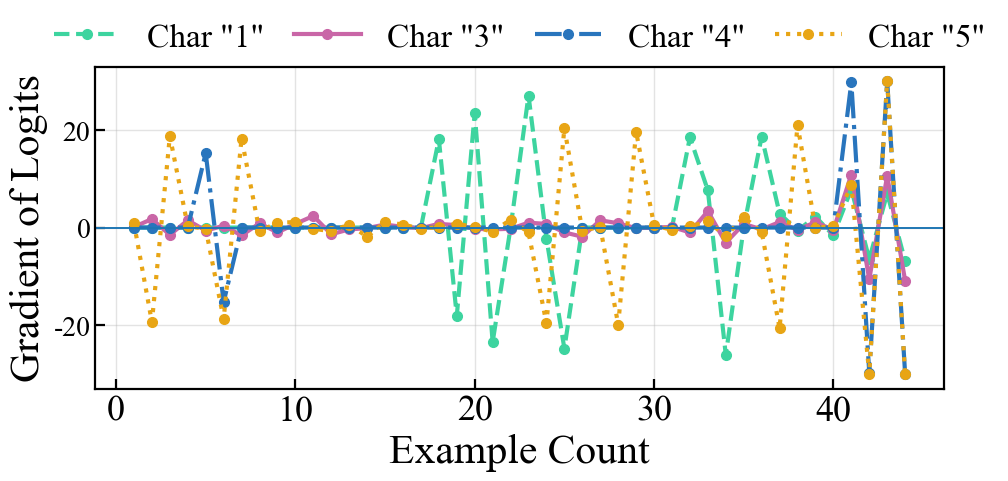
    }
    \caption{\textbf{Discrete gradients of log probability.} (\texttt{Qwen-235B-A22B-Instruct})
    (a) Equivalent instructions yield small and stable gradients. (b) Incorrect but
    consistent instructions show gradual stabilization.(c) Incorrect and inconsistent
    instructions produce large and irregular gradients.}
    \label{fig:add_logit_235b_gradient}
\end{figure}

\begin{figure}[htbp]
    \centering
    \includegraphics[width=0.32\textwidth]{
        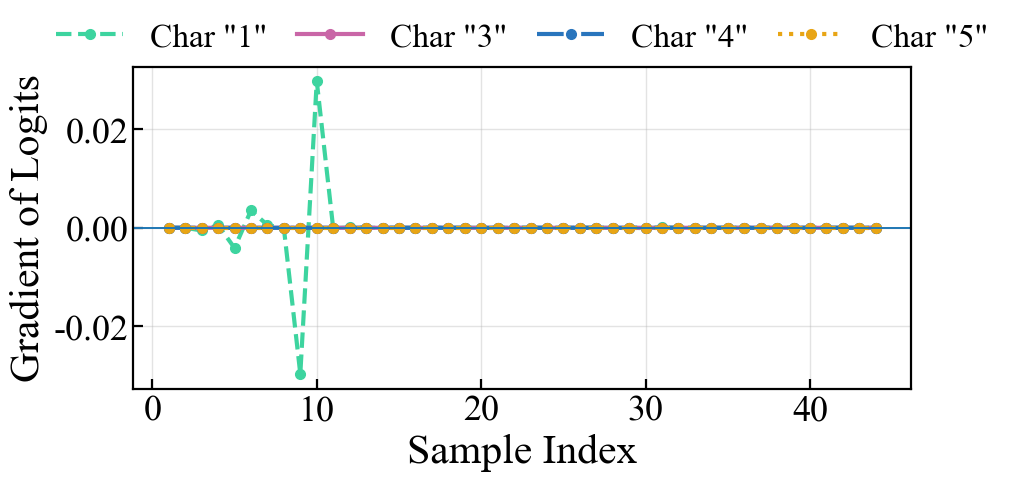
    }
    \includegraphics[width=0.32\textwidth]{
        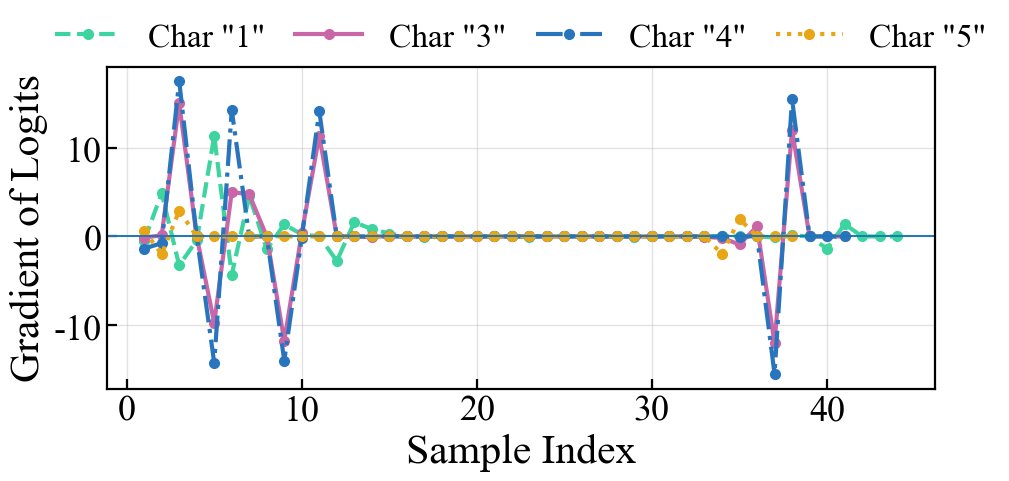
    }
    \includegraphics[width=0.32\textwidth]{
        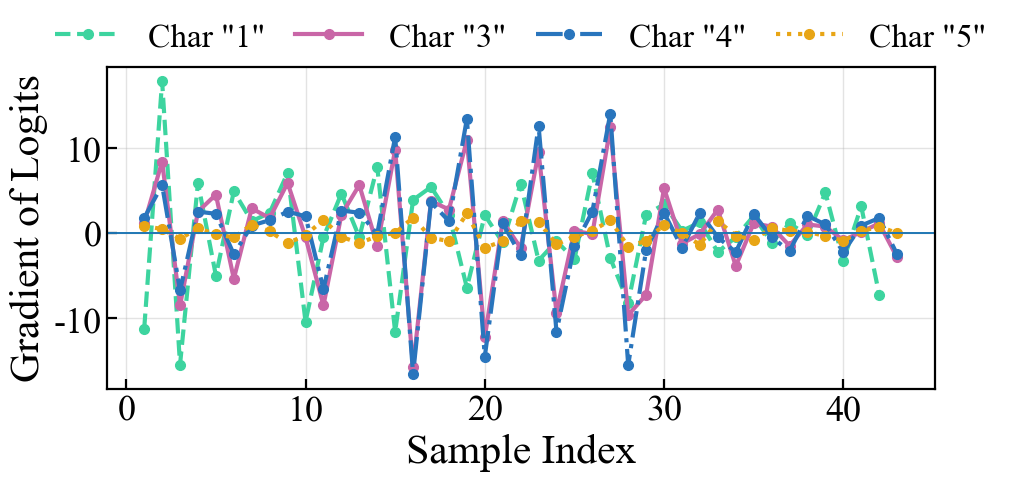
    }
    \caption{\textbf{Discrete gradients of log probability.} (\texttt{Qwen-30B-A3B-Instruct})
    (a) Equivalent instructions yield small and stable gradients. (b) Incorrect but
    consistent instructions show gradual stabilization.(c) Incorrect and inconsistent
    instructions produce large and irregular gradients.}
    \label{fig:add_logit_a3b_gradient}
\end{figure}

\subsection{CoT Effect}
\label{app:cot_effect} We further study the impact of CoT prompting on \texttt{Llama3.2-3B},
focusing on how the intermediate steps of the reasoning chain affects
performance. We consider two CoT variants for the same multi-digit addition problem,
which lead to the same correct final answer but differ in how intermediate steps
are constructed.

\paragraph{In-distribution decomposition CoT.}
For the \emph{in-distribution decomposition} CoT variant, each intermediate reasoning
step corresponds to a well-defined subtask (e.g., digit-wise decomposition and
carry propagation). To ensure that the model has intrinsic capability for these
intermediate steps, we generate $10{,}000$ training examples for \emph{each}
step and fine-tune the model on these step-level datasets prior to evaluation. Each
example is formatted as an instruction--input--output triple, analogous to the
pretraining setup in \cref{app:pretrain_error}.
\begin{quote}
    \small \textbf{Step 1:} $1234561 = 1234560 + 1$ \\
    \textbf{Step 2:} $1323212 = 1323210 + 2$ \\
    \textbf{Step 3:} $1234560 + 1323210 \rightarrow (123456 + 132321)\times 10$
    \\
    \textbf{Step 4:} $1 + 2$ \\
    \textbf{Step 5:} Combine results.
\end{quote}
Each step corresponds to a sub-computation the model can reliably perform in isolation.

\paragraph{Out-of-distribution decomposition CoT.}
For the \emph{out-of-distrubution} CoT variant, the final answer remains correct,
but the intermediate steps are constructed in a form that the model has \emph{not}
been explicitly trained on. In this case, no additional fine-tuning is performed
for the intermediate steps, and the model is required to follow the reasoning chain
purely through prompting.
\begin{quote}
    \small \textbf{Step 1:} $1234561 = 1000000 + 234561$ \\
    \textbf{Step 2:} $1323212 = 1000000 + 323212$ \\
    \textbf{Step 3:} Add corresponding components. $1000000 + 1000000 = 2000000,
    234561 + 323212 = 557773$ \\
    \textbf{Step 4:} Combine results.
\end{quote}
Although logically correct, these intermediate steps do not correspond to well-supported
arithmetic routines in the model’s pretraining distribution and provide little useful
guidance for subsequent inference.

\paragraph{Training Setup.}
All fine-tuning for the good CoT variant uses the same LoRA configuration, optimizer,
and training hyperparameters as in \cref{app:pretrain_error}. In particular, the
architecture, learning rate, number of epochs, batch size, and quantization
settings are kept fixed. This isolates the effect of step-level intrinsic capability
from other confounding factors.

\paragraph{Results.}
We find that the grounded chain substantially improves accuracy over direct
prompting without CoT, whereas the ungrounded chain provides no benefit and can even
degrade performance (Figure~\ref{fig:chain}). This indicates that CoT prompting
is effective only when each reasoning step is well-learned in pretrained sub-tasks,
supporting our view of CoT as a sequence of structured in-context updates rather
than an arbitrary logical derivation.